\documentclass{article}
\usepackage{iclr2026_conference,times}

\usepackage{microtype}
\usepackage{graphicx}
\usepackage{subcaption}

\usepackage{hyperref}
\usepackage{url}

\newcommand{\appref}[1]{Appendix~\ref{#1}}

\usepackage{amsmath}
\usepackage{amssymb}
\usepackage{mathtools}
\usepackage{amsthm}
\usepackage{multirow}
\usepackage{diagbox}
\usepackage{makecell}

\usepackage{soul}
\usepackage{booktabs}
\usepackage{latexsym}
\usepackage[T1]{fontenc}

\usepackage{nicefrac}       
\usepackage{newunicodechar}
\usepackage{xcolor}         
\usepackage{colortbl}
\usepackage{tcolorbox}
\tcbuselibrary{breakable}
\usepackage{enumitem}
\usepackage{fontawesome5}

\definecolor{csfail}{HTML}{C0392B}
\definecolor{csok}{HTML}{1E8449}
\definecolor{csobsc}{HTML}{7F8C8D}
\definecolor{csactbg}{HTML}{EEF2F7}
\definecolor{csactfg}{HTML}{1A2A3A}
\definecolor{cspivotc}{HTML}{E67E22}

\newtcolorbox{cstraj}[2]{%
    breakable,
    colback=white, colframe=#1,
    boxrule=0.9pt, arc=2.5pt, left=6pt, right=6pt, top=4pt, bottom=4pt,
    fonttitle=\bfseries\small, coltitle=white, colbacktitle=#1,
    title={#2}
}

\newcommand{\csobs}[1]{{\footnotesize\color{csobsc}\textsf{obs:}~#1}\par}
\newcommand{\csthink}[1]{{\small\textsf{think:}~``#1''}\par}
\newcommand{\csact}[1]{\smallskip\noindent\colorbox{csactbg}{\color{csactfg}\small\ttfamily #1}\par\smallskip}
\newcommand{\csnote}[2]{\noindent\hspace{1.2em}{\small\color{#1}\faArrowRight~\itshape #2}\par}
\newcommand{\csstep}[2][]{\smallskip\noindent{\small\bfseries S#2}~#1~}
\providecommand{\say}[1]{``#1''}

\usepackage{upquote}
\usepackage{listings}
\lstdefinestyle{promptstyle}{%
    basicstyle=\fontfamily{lmtt}\selectfont\scriptsize,
    columns=fullflexible, keepspaces=true,
    breaklines=true, breakatwhitespace=false,
    breakindent=0pt, breakautoindent=false, postbreak=\mbox{},
    upquote=true,
    aboveskip=2pt, belowskip=0pt, xleftmargin=0pt,
}
\definecolor{cspromptc}{HTML}{34495E}
\newtcolorbox{promptcard}[1]{%
    breakable,
    colback=white, colframe=cspromptc,
    boxrule=0.9pt, arc=2.5pt, left=6pt, right=6pt, top=4pt, bottom=4pt,
    fonttitle=\bfseries\small, coltitle=white, colbacktitle=cspromptc,
    title={#1}
}
\newcommand{\promptbox}[2]{%
    \begin{promptcard}{#1}
    \lstinputlisting[style=promptstyle]{#2}
    \end{promptcard}
}

\usepackage[capitalize,noabbrev]{cleveref}

\usepackage{aliascnt}
\theoremstyle{plain}
\newtheorem{theorem}{Theorem}[section]

\newaliascnt{proposition}{theorem}
\newtheorem{proposition}[proposition]{Proposition}
\aliascntresetthe{proposition}
\newaliascnt{lemma}{theorem}
\newtheorem{lemma}[lemma]{Lemma}
\aliascntresetthe{lemma}
\newaliascnt{corollary}{theorem}
\newtheorem{corollary}[corollary]{Corollary}
\aliascntresetthe{corollary}

\theoremstyle{definition}
\newaliascnt{definition}{theorem}

\aliascntresetthe{definition}
\newaliascnt{assumption}{theorem}
\newtheorem{assumption}[assumption]{Assumption}
\aliascntresetthe{assumption}

\theoremstyle{remark}
\newaliascnt{remark}{theorem}
\newtheorem{remark}[remark]{Remark}
\aliascntresetthe{remark}

\crefname{proposition}{Proposition}{Propositions}
\crefname{lemma}{Lemma}{Lemmas}
\crefname{corollary}{Corollary}{Corollaries}
\crefname{definition}{Definition}{Definitions}
\crefname{assumption}{Assumption}{Assumptions}
\crefname{remark}{Remark}{Remarks}

\usepackage[disable]{todonotes}

\iclrfinalcopy

\newcommand{\Envelope}{\raisebox{0.4pt}{\scriptsize\faIcon[regular]{envelope}}}

\usepackage{xspace}
\newcommand{\method}{CAST\xspace}

\title{CAST: Game Solvers as Turn-Level Teachers for LLM Agents}

\author{%
Yu Wang\textsuperscript{1,4*}, 
Yi-Kai Zhang\textsuperscript{2,4*}, 
Wentao Shi\textsuperscript{1\ \Envelope},
Ziang Ye\textsuperscript{1,4}, 
Yuchun Miao\textsuperscript{3,4}, 
Yueqing Sun\textsuperscript{4} \\[2pt]
\bfseries
Qi Gu\textsuperscript{4\ \Envelope},
Xunliang Cai\textsuperscript{4}, 
Lan-Zhe Guo\textsuperscript{2}, 
Han-Jia Ye\textsuperscript{2}, 
Fuli Feng\textsuperscript{1} \\[6pt]
\textsuperscript{1}University of Science and Technology of China\quad
\textsuperscript{2}Nanjing University\quad
\textsuperscript{3}Wuhan University\\
\textsuperscript{4}Meituan, China \\[4pt]
}

\begin{document}

\AddToShipoutPictureFG*{%
    \AtTextUpperLeft{%
        \makebox[\textwidth][r]{%
        \raisebox{6mm}{%
            \includegraphics[height=0.8cm]{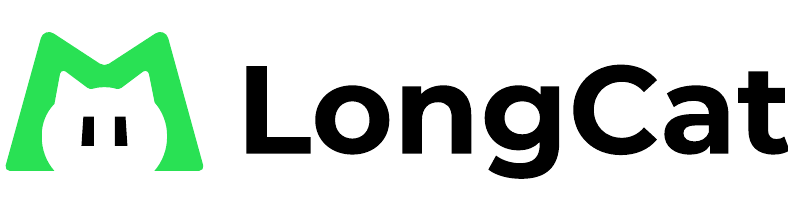}%
        }%
        }%
    }%
}

\maketitle

\lhead{CAST: Game Solvers as Turn-Level Teachers for LLM Agents}

\rhead{\ifnum\value{page}>1
    \raisebox{-2.5pt}[0pt][0pt]{\includegraphics[height=0.5cm]{figures/logos/longcat.pdf}}%
\fi}

\begingroup
    \renewcommand{\thefootnote}{*}
    \footnotetext{Equal contribution.}
\endgroup
\begingroup
    \renewcommand{\thefootnote}{\Envelope}
    \footnotetext{Corresponding authors: \texttt{guqi03@meituan.com} and \texttt{shiwentao123@mail.ustc.edu.cn}.}
\endgroup
\renewcommand{\thefootnote}{\arabic{footnote}}

\begin{abstract}
Training large language models (LLMs) to act in long-horizon games is a promising step toward generalist decision-making, yet reinforcement learning with verifiable rewards (RLVR) relies on sparse final rewards that reveal little about which decisions determine success. Denser process signals could supply this missing turn-level credit, but existing sources are hard to keep both cheap and accurate. We observe that changes in a game solver's state value reveal whether an action advances the state toward success. Building on this insight, we propose \textbf{CAST} (\textbf{C}redit \textbf{A}ssignment from \textbf{S}olver \textbf{T}eachers), which converts these value changes into solver advantages and injects them into RLVR as turn-level signals. We further show that, under a soft-optimal solver assumption, maximizing the solver advantage is equivalent to on-policy distillation from the solver, requiring only scalar values rather than teacher logits. Across Sokoban, Minesweeper, and Rush Hour, CAST outperforms all trained baselines on every game under both in-domain and unseen-difficulty evaluation and achieves the highest average zero-shot performance on ALFWorld and WebShop. Our code is available at \begingroup\hypersetup{pdfborder={0 0 0}}\href{https://github.com/Wloner0809/CAST}{\faGithub~\texttt{github.com/Wloner0809/CAST}}\endgroup.
\end{abstract}

\section{Introduction}
\label{sec:introduction}

As large language models (LLMs) and multimodal foundation models~\citep{Qwen3, longcat2601} move from passive generation toward active decision-making, a central goal is to build generalist agents that can make decisions in embodied and open-world environments~\citep{pi05}. Such agents need to act in evolving states, where the effects of early decisions may emerge only later and are often difficult to reverse~\citep{ye2026look}. These settings require a broad range of capabilities, including long-horizon planning, goal-directed exploration, and recovery from mistakes. Classical games such as Sokoban and Minesweeper offer an ideal testbed for these capabilities~\citep{junghanns2001sokoban, mnih2015human, lmgame-Bench}, with well-defined rules, verifiable feedback, and scalable decision-making environments. In fact, such games can already be solved efficiently by domain-specific solvers, ranging from heuristic search and dynamic programming to specialized DQN value networks~\citep{DQN}. General-purpose LLMs, by contrast, remain unreliable even in these structured environments, as shown in \autoref{fig:teaser} (left). Games therefore provide a bridge for studying how a general-purpose LLM can become a reliable interactive agent\citep{GameWorld, BALROG}.

A central bottleneck in training generalist LLM game agents is the lack of scalable, fine-grained learning signals. Reinforcement Learning with Verifiable Rewards (RLVR) has been effective in single-turn tasks~\citep{V0, chen2026learning, DeepSeek-R1}, but in games, final rewards are often too sparse. As illustrated in \autoref{fig:teaser} (right), such outcome signals cannot reveal which decisions in a trajectory cause the final win or loss, creating a fundamental credit-assignment challenge. To address this challenge, existing approaches obtain denser process signals through costly search~\citep{RAP}, learned process reward models~\citep{Agentprm}, or cross-trajectory comparisons such as GiGPO~\citep{GiGPO}, yet still face trade-offs among computation, supervision, and signal reliability. We observe that game solvers can provide this missing signal, as a solver can evaluate an action from the state at which it was taken. SFT on solver-generated expert trajectories, however, does not fully exploit this capability: it exposes the LLM only to expert-visited states and offers little guidance once the model deviates during interaction. This motivates an on-policy RL paradigm in which the LLM explores under its own policy and the solver acts as a \emph{turn-level teacher}.

Aligning the model's on-policy exploration with immediate solver feedback is, in essence, a form of on-policy distillation (OPD). Existing LLM distillation methods~\citep{OPD, Minillm}, however, typically rely on teacher logits over the full token space, whereas classical solvers return only an optimal action or a scalar cost-to-go, not such a distribution. We therefore propose \textbf{CAST} (\textbf{C}redit \textbf{A}ssignment from \textbf{S}olver \textbf{T}eachers), a distillation strategy built directly on solvers. Since a solver can complete a game from an intermediate state, it assigns every state a \textit{state value}: how close that state is to winning. For each action sampled by the LLM, CAST compares the solver values immediately before and after the action to obtain a solver advantage, and injects it into RLVR as turn-level signals. This converts the outcome reward into process supervision indicating whether the current action moves the state closer to success, at negligible training overhead. Our theoretical analysis shows that, for a sufficiently strong solver, an action's log-probability under the solver's implicit action distribution is strictly proportional to its one-step change in state value. Consequently, increasing the likelihood of actions with larger solver advantages is mathematically equivalent to on-policy distillation that aligns the LLM's action distribution with the solver's. This equivalence resolves the obstacle above: the scalar advantage already carries the teacher's action preference, so a single scalar suffices for distillation without the teacher's full action distribution, yielding an efficient logit-free distillation.

\begin{figure*}[t]
    \centering
    \includegraphics[width=\textwidth]{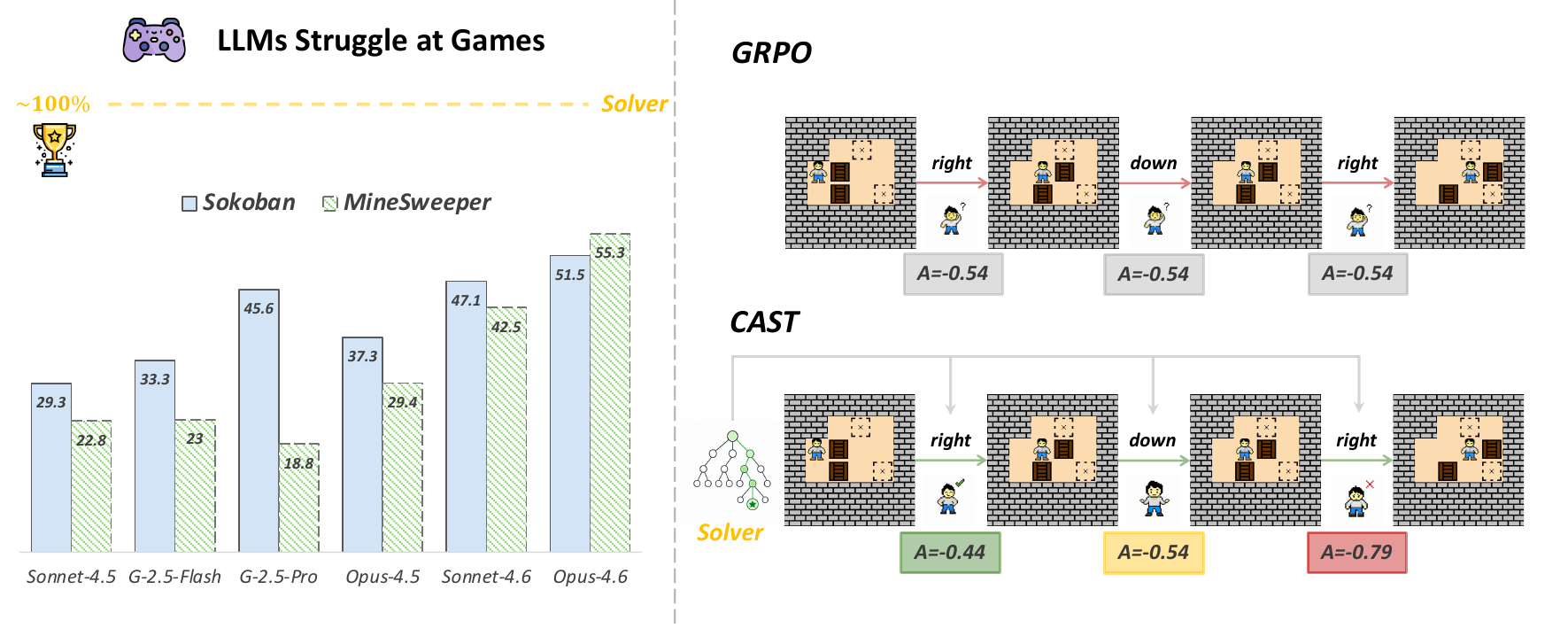}
    \caption{\textbf{Overview.} \textbf{(Left)} Several current closed-source LLMs struggle on classical games. \textbf{(Right)} Outcome-only RLVR lacks fine-grained feedback, whereas our solver-derived signal provides accurate turn-level credit.}
    \label{fig:teaser}
\end{figure*}

In practice, the absolute scale of solver signals varies across games and often contains extreme values that destabilize training. We address this with two lightweight shaping steps. First, an \texttt{asinh} transformation smoothly compresses extreme values while preserving resolution in the small-signal regime; a theoretical analysis shows this amounts to a robustified KL constraint. Second, batch-level root-mean-square (RMS) normalization removes cross-domain scale differences. These let the solver advantage serve as a stable, comparable process signal that integrates seamlessly into RLVR.

We evaluate \method on Sokoban, Minesweeper, and Rush Hour, which span complementary challenges in long-horizon planning, partial-observation inference, and constrained combinatorial search. \method achieves the best performance among all trained methods across every game under both in-domain and unseen-difficulty settings, consistently outperforming outcome-only RLVR baselines and the process-level baseline GiGPO~\citep{GiGPO}. It also reaches DAPO's~\citep{DAPO} peak validation performance in substantially fewer training steps and transfers zero-shot to held-out ALFWorld~\citep{Alfworld} and WebShop~\citep{Webshop} domains without further fine-tuning. Ablations validate the solver-advantage weighting and signal transformations; further analyses show negligible solver overhead and that a learned value network retains much of the benefit of exact solver guidance.

In summary, our contributions are as follows:
\begin{itemize}[noitemsep,topsep=0pt,leftmargin=*]
    \item \textbf{Solver-derived credit.} We cast game solvers as turn-level teachers and define a solver advantage that assigns fine-grained credit to LLM-sampled actions in RLVR.
    \item \textbf{Logit-free distillation.} We prove that maximizing this advantage is equivalent to on-policy distillation, giving a distillation objective needing no teacher logits, and stabilize it with \texttt{asinh} compression and batch-level RMS normalization.
    \item \textbf{Performance and generalization.} \method reaches state-of-the-art results on all three games in-domain and at unseen difficulties, improves sample efficiency, and transfers zero-shot to ALFWorld and WebShop.
    \item \textbf{Practicality.} An exact solver adds negligible overhead, and a learned value network provides comparable performance when no exact solver exists.
\end{itemize}

\section{Method}
\label{sec:method}

Our method lets a game-specific solver act as a turn-level teacher that scores each action the LLM takes, refining the sparse terminal reward into a dense per-step process signal. \autoref{sec:preliminaries} formalizes games as reinforcement learning problems and identifies why outcome-only RLVR is limited. \autoref{sec:solver_advantage} introduces the solver teacher, constructs its per-step score, gives the full pipeline that makes it stable and integrates it into training, and finally explains why this score is equivalent to on-policy distillation (OPD, \autoref{thm:main}).

\subsection{Preliminaries}
\label{sec:preliminaries}

\paragraph{Games as multi-turn MDPs.}
We formulate each game as a finite-horizon Markov decision process (MDP) $\mathcal{M}=(\mathcal{S},\mathcal{A},P,r,H)$ with horizon $H$. At turn $t$, the LLM policy $\pi_\theta$ observes a textual rendering of the state $s_t\in\mathcal{S}$, samples an action $a_t\sim\pi_\theta(\cdot\mid s_t)$, and the environment transitions via $P(s_{t+1}\mid s_t,a_t)$ until the episode terminates at some step $T\le H$, yielding a trajectory $\tau=(s_0,a_0,\ldots,s_T)$. For Sokoban and Rush Hour, $s_t$ is the complete board; for Minesweeper, $s_t$ is the \emph{information state} (the revealed board plus the constraints it induces on hidden mines), which is a sufficient statistic that keeps the process Markov. The reward is sparse and verifiable:
\begin{equation}
    r_t=0\quad(t<T),\qquad
    R(\tau)=r_T=\mathbf{1}\{s_T\in\mathcal{S}_{\text{solved}}\}.
    \label{eq:sparse_reward}
\end{equation}
This $0/1$ terminal reward is the true optimization objective throughout the paper; every solver signal introduced later is auxiliary supervision and does not change what counts as winning.

\paragraph{Outcome-supervised RLVR and its bottleneck.}
RLVR maximizes the expected return $J(\theta)=\mathbb{E}_{\tau\sim\pi_\theta}[R(\tau)]$ via policy gradients~\citep{PolicyGradient}. To avoid a learned critic, GRPO~\citep{GRPO} samples a group of $G$ trajectories from $\pi_{\theta_{\text{old}}}$ for a prompt $q$ and normalizes their returns into a trajectory-level advantage,
\begin{equation} \small
    \hat{A}^{\text{outcome}}_i
    =
    \frac{R_i-\mu_R}{\sigma_R+\delta},
    \qquad
    \mu_R=\frac{1}{G}\sum_{j=1}^{G}R_j,\quad
    \sigma_R=\sqrt{\frac{1}{G}\sum_{j=1}^{G}(R_j-\mu_R)^2},
    \label{eq:outcome_advantage}
\end{equation}
with a small constant $\delta>0$ for numerical stability. This trajectory-level scalar is assigned to every token of the trajectory and optimized through a clipped surrogate objective, where the importance-sampling ratio $\rho_t$ between $\pi_\theta$ and $\pi_{\theta_{\text{old}}}$ is clipped to $[1{-}\epsilon,1{+}\epsilon]$:
\begin{equation} \small
    \mathcal{J}_{\mathrm{GRPO}}(\theta)
    =
    \mathbb{E}_{q,\,\{o_i\}\sim\pi_{\theta_{\text{old}}}}
    \!\Bigg[\frac{1}{G}\sum_{i=1}^{G}\frac{1}{|o_i|}
    \sum_{t=1}^{|o_i|}
    \Big(\min\!\big(\rho_t\,\hat{A}_{i},\,
    \operatorname{clip}(\rho_t,1{-}\epsilon,1{+}\epsilon)\,\hat{A}_{i}\big)
    -\beta_{\text{kl}}\,D_{\mathrm{KL}}\!\big[\pi_\theta\,\|\,\pi_{\text{ref}}\big]
    \Big)\Bigg].
    \label{eq:grpo_objective}
\end{equation}
The bottleneck is now explicit: because $\hat{A}_{i} = \hat{A}^{\text{outcome}}_i$ is computed solely from the terminal outcome, every turn in a trajectory receives the same credit. This coarse trajectory-level credit is the root of the credit-assignment failure in long-horizon games, motivating the turn-level solver signal introduced next.

\subsection{Solver-Guided Turn-Level Credit}
\label{sec:solver_advantage}

\begin{figure*}[t]
    \centering
    \includegraphics[width=0.94\textwidth]{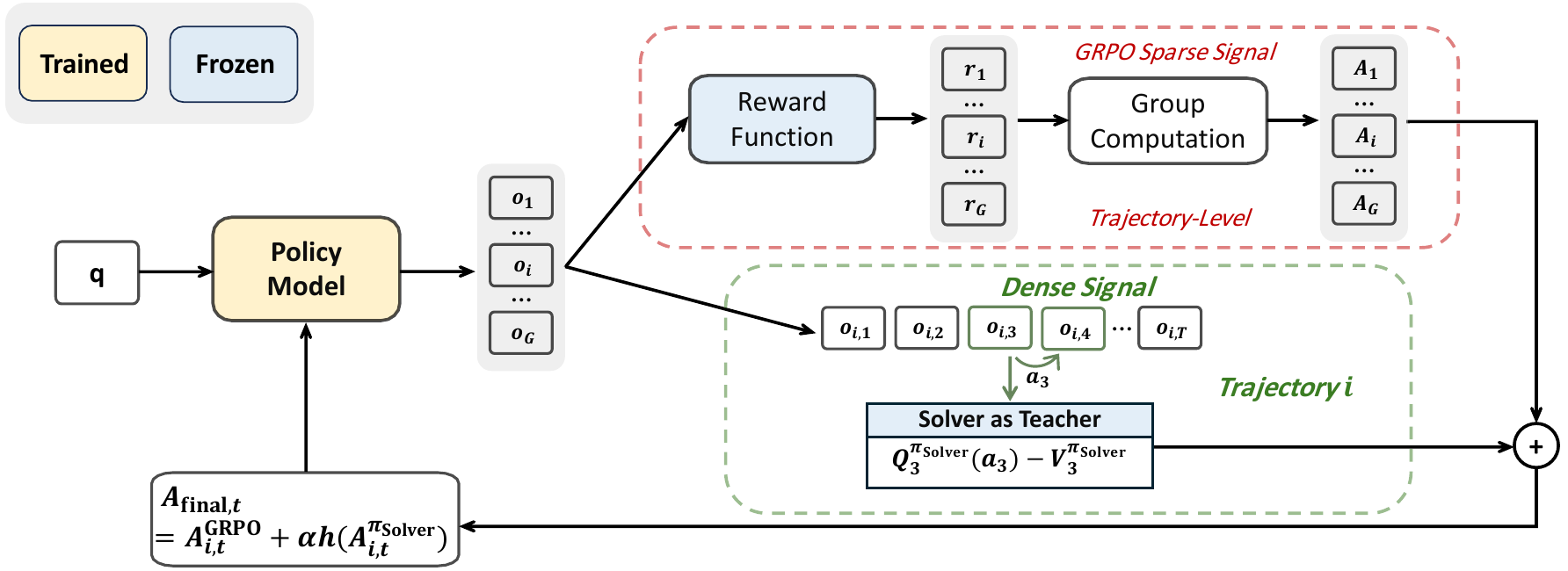}
    \caption{\textbf{Method overview.} We augment GRPO's outcome advantage with a shifted solver advantage derived from turn-level cost-to-go changes.}
    \label{fig:method}
\end{figure*}

\paragraph{A per-step score from the solver.}
To score each action individually, we need a way to measure how good the current state is. The $0/1$ terminal reward cannot do this, as it only takes a value at the very end. A solver can, because it completes a game from any state. This property equips any solver with a general-purpose measure that we call the \emph{cost-to-go} $N(s)$: the minimum amount of work needed to reach a win from state $s$. Its exact unit varies by game, but the meaning of \say{how far from the goal} is universal: for Sokoban and Rush Hour, $N(s)$ is the minimum number of actions to solve the board, and for Minesweeper it is the minimum number of reveals needed to clear all safe cells.

A natural per-step score is $N(s_t)-N(s_{t+1})$, the drop in cost-to-go caused by the move. We ground this score in RL by defining an auxiliary shortest-path objective ($-1$ per action, terminal value $0$) and setting $V^{\pi_{\text{Solver}}}(s)=-N(s)$, which makes states closer to the goal more valuable.
\begin{equation}
\begin{aligned}
    Q^{\pi_{\text{Solver}}}(s_t,a_t)
    &=-1+\mathbb{E}_{s_{t+1}}[V^{\pi_{\text{Solver}}}(s_{t+1})], \,\, \text{and} \, V^{\pi_{\text{Solver}}}(s_t)=-N(s_t) \\
    A^{\pi_{\text{Solver}}}(s_t,a_t)
    &=-1+N(s_t)-\mathbb{E}_{s_{t+1}}[N(s_{t+1})].
\end{aligned}
    \label{eq:solver_advantage}
\end{equation}
The advantage $A=Q-V$ measures how much better a specific action is compared to the expected outcome under the solver's policy. The $\mathbb{E}_{s_{t+1}}$ is over the environment transition; for the deterministic games studied here it reduces to the single resulting state, and we keep the expectation for generality.

\paragraph{A shift so that progress maps to positive credit.}
Under the solver's optimal policy, the expected next-state cost $\mathbb{E}_{s_{t+1}}[N(s_{t+1})]$ equals $N(s_t)-1$, because the solver always reduces the remaining steps by exactly one. Substituting into \autoref{eq:solver_advantage} gives $A^{\pi_{\text{Solver}}}=-1+(N(s_t)-(N(s_t)-1))=0$: an optimal action has zero advantage, and any suboptimal action that fails to reduce the cost-to-go by a full step scores negative. This makes the raw advantage non-positive, which is inconvenient as a signal. Shifting by $+1$ gives the \emph{shifted solver advantage},
\begin{equation}
    \widetilde{A}^{\pi_{\text{Solver}}}(s_t,a_t)
    =A^{\pi_{\text{Solver}}}(s_t,a_t)+1
    =N(s_t)-\mathbb{E}_{s_{t+1}}[N(s_{t+1})],
    \label{eq:shifted_solver_advantage}
\end{equation}
which equals the drop in cost-to-go caused by the move. Its meaning is direct: an optimal move that advances one step receives $+1$, a move with no progress receives $0$, and a harmful move receives negative credit. The one special case is a transition into an unsolvable dead state, where $N(s_{t+1})=\infty$; we cap such a transition at the finite penalty $-N(s_t)$, counting it as losing the $N(s_t)$ steps that would otherwise have won the game.

\paragraph{Shaping the signal for training.}
The shifted advantage is clean in theory but problematic if used directly. First, it is heavy-tailed: most moves change the cost-to-go by only $0$ or $\pm 1$, but the dead-state penalty $-N(s_t)$ can be large, and a few such extremes would dominate the gradient. Second, its scale varies across games, as a hard board has a cost-to-go of dozens while an easy one has only a few. We address both with two lightweight transformations. First, an \texttt{asinh} compression,
\begin{equation}
    g(x)=\operatorname{asinh}(x)=\ln\!\big(x+\sqrt{x^2+1}\,\big),
    \label{eq:asinh}
\end{equation}
which is near-linear for small $x$, so common small-progress values are preserved and remain distinguishable, and grows only logarithmically for large $|x|$, so rare dead-state penalties are compressed. We use \texttt{asinh} rather than a plain logarithm because it is defined and odd-symmetric across negative, zero, and positive values, fully preserving the sign that encodes beneficial, neutral, and harmful moves. Second, batch-level RMS normalization rescales the signal by its root-mean-square magnitude over all turns in the batch,
\begin{equation}
    h(x)=\frac{g(x)}{\mathrm{RMS}_{\mathcal{B}}(g)+\epsilon},
    \qquad
    \mathrm{RMS}_{\mathcal{B}}(g)=\sqrt{\tfrac{1}{|\mathcal{B}|}\textstyle\sum_{(i,t)\in\mathcal{B}} g\!\left(\widetilde{A}^{\pi_{\text{Solver}}}(s_{i,t},a_{i,t})\right)^2},
    \label{eq:rms_norm}
\end{equation}
where $\mathcal{B}$ contains all turns in the batch, placing the signal on a consistent scale across games. We divide by the RMS \emph{without} subtracting the mean, which is deliberate: the value $0$ should keep meaning \say{no progress} so that positive stays beneficial and negative stays harmful; subtracting the mean would shift this zero point and destroy the sign information on which the signal relies.

\paragraph{Combining with GRPO.}
We combine this per-step process signal with the terminal reward, so as to keep the true objective of winning while adding fine-grained per-step credit (\autoref{fig:method}). For trajectory $i$ at turn $t$, we write $\widetilde{A}^{\pi_{\text{Solver}}}_{i,t}\triangleq\widetilde{A}^{\pi_{\text{Solver}}}(s_{i,t},a_{i,t})$ for the shifted solver advantage evaluated at that trajectory's state-action pair. We add the shaped solver signal to GRPO's outcome advantage and broadcast the result to every token of the turn,
\begin{equation}
    \hat{A}_{i,t}
    =
    \hat{A}^{\text{outcome}}_i + \alpha\, h\!\left(\widetilde{A}^{\pi_{\text{Solver}}}_{i,t}\right),
    \label{eq:final_advantage}
\end{equation}
with a single coefficient $\alpha$ controlling the strength of solver guidance. Three granularities meet here: $\hat{A}^{\text{outcome}}_i$ is trajectory-level and identical across all turns of a trajectory, anchoring the global win-or-loss credit; $h(\widetilde{A}^{\pi_{\text{Solver}}}_{i,t})$ is turn-level and varies step by step, refining credit within the trajectory; and their sum is broadcast to the token level, since GRPO optimizes over tokens. The combined advantage $\hat{A}_{i,t}$ replaces $\hat{A}^{\text{outcome}}_i$ in the clipped GRPO surrogate of \autoref{eq:grpo_objective}, so the terminal reward still anchors trajectory-level credit while the solver signal refines it turn by turn.

\paragraph{Equivalence to logit-free OPD.}
Having presented the complete pipeline, we now show that the solver advantage is not merely a hand-designed signal but is mathematically equivalent to distilling from the solver. The key observation is that a sufficiently strong solver can be viewed as a \emph{soft-optimal} teacher policy whose action probabilities grow exponentially with the action value: $\pi_{\text{Solver}}(a\mid s)\propto \exp\!\big(Q^{\pi_{\text{Solver}}}(s,a)/\tau\big)$ for a temperature $\tau>0$~\citep{MaxEntIRL,SAC}. Taking $\log$ immediately gives $A^{\pi_{\text{Solver}}}(s,a)=\tau\,\log\pi_{\text{Solver}}(a\mid s)$: the solver advantage is exactly the teacher's log-preference for the action. We then obtain the following result (proved in \appref{app:kl_proof}):

\begin{tcolorbox}[
    colback=gray!3,
    colframe=black,
    boxrule=0.5pt,
    arc=2pt,
    left=6pt,
    right=6pt,
    top=4pt,
    bottom=4pt,
    breakable
]
\begin{theorem}[Implicit Objective of CAST]
\label{thm:main}
Assume the solver is soft-optimal with temperature $\tau>0$, the solver advantages satisfy $|A^{\pi_{\text{Solver}}}|\lesssim 1$, and GRPO provides an unbiased task-return gradient under a frozen-visitation surrogate. Then the policy-gradient update of the training rule~\autoref{eq:final_advantage} equals the gradient of
\begin{equation}
    \mathcal{J}(\theta)
    =\underbrace{\mathbb{E}_{s_0\sim\mu}\!\big[V^{\pi_\theta}_{\text{task}}(s_0)\big]}_{\text{task return}}
    -\;\beta\,\underbrace{\mathbb{E}_{s\sim d^{\pi_\theta}}\!\Big[\mathrm{H}\!\big(\pi_\theta(\cdot|s),\,\pi_{\text{Solver}}(\cdot|s)\big)\Big]}_{\text{cross-entropy distillation}},
    \,\,\, \beta=\frac{\alpha\,\tau}{\sqrt{2}\,(\mathrm{RMS}_{\mathcal{B}}(g)+\epsilon)}\,,
    \label{eq:implicit_obj}
\end{equation}
where $\mathrm{H}(\pi_\theta,\pi_{\text{Solver}})=-\mathbb{E}_{a\sim\pi_\theta}[\log\pi_{\text{Solver}}(a\mid s)]$ is the cross-entropy from $\pi_\theta$ to $\pi_{\text{Solver}}$, and $d^{\pi_\theta}(s)$ is the expected number of times state $s$ is visited during an episode.
\end{theorem}
\end{tcolorbox}

\noindent The theorem reveals that our method implicitly maximizes the task return while minimizing the cross-entropy from $\pi_\theta$ to $\pi_{\text{Solver}}$, which is the objective of OPD~\citep{OPD}. We have:

\emph{Logit-free OPD.}
Traditional distillation requires the teacher's full output distribution. Here, the identity $A^{\pi_{\text{Solver}}}=\tau\log\pi_{\text{Solver}}$ means a single scalar per action already encodes the teacher's log-preference, that is, no logits needed.

\emph{Why the student can surpass the solver.}
Isolating the KL component of~\autoref{eq:implicit_obj} and solving the per-state optimization yields a closed-form optimal policy $\pi^*(a|s)\propto\pi_{\text{Solver}}(a|s)\,\exp\!\big(A^{\pi_\theta}_{\text{task}}(s,a)/\beta\big)$: the solver distribution serves as a prior, and the task advantage tilts it exponentially. At $\beta\to\infty$ (pure distillation), $\pi^*\to\pi_{\text{Solver}}$; at finite $\beta$, the tilt lets $\pi^*$ deviate wherever the task reward warrants it, so the student can \emph{surpass} the teacher.
\section{Experiments}
\label{sec:experiments}

Our experiments address four research questions below:

\textbf{RQ1:} Does the solver-derived turn-level signal outperform outcome-only and process-level RL baselines in-domain on the training games (\autoref{sec:exp_main})?

\textbf{RQ2:} Do the trained agents generalize to unseen difficulty levels within the training games and to held-out domains (\autoref{sec:exp_main}, \autoref{sec:exp_ood})?

\textbf{RQ3:} Does each design component contribute to the final performance, namely the solver-advantage weight, the \texttt{asinh} transformation, and batch-level RMS normalization (\autoref{sec:exp_ablation})?

\textbf{RQ4:} Is solver guidance practical, i.e., does it add little training overhead and remain effective when the solver is a learned value network (\autoref{sec:exp_analysis})?

\subsection{Experimental Setup}

\paragraph{Datasets and Benchmarks}

We train and evaluate on three classic games, \textbf{Sokoban}, \textbf{Minesweeper}, and \textbf{Rush Hour}, whose clear rules and efficient solvers provide verifiable rewards and process-level supervision. Together, they span complementary reasoning challenges, from long-horizon planning and partial-observation inference to constrained combinatorial search. For each game, we procedurally generate instances across difficulty levels defined by game-specific structural parameters and calibrated with solver-derived solution complexity. We canonicalize board configurations to remove duplicates, prevent overlap between training and test instances, and verify solvability with the corresponding solver. Evaluation covers two within-game settings: \emph{In-Domain} (ID), covering difficulty levels seen during training, and \emph{Unseen-Difficulty}, covering harder levels held out from training. We evaluate each setting on 200 held-out instances per game. Throughout, \emph{unseen difficulty} denotes this within-game shift, whereas \emph{OOD transfer} denotes zero-shot evaluation on a held-out task domain. Full generation and difficulty-control details are given in~\appref{app:dataset_details}.

To probe OOD transfer beyond the training games, we further evaluate zero-shot on two held-out agentic domains without any further fine-tuning: the embodied benchmark ALFWorld~\citep{Alfworld} and the web benchmark WebShop~\citep{Webshop}; benchmark sources and evaluation protocols are detailed in~\appref{app:benchmark_details}.

\paragraph{Baselines and Models}

We use \textit{Qwen3-4B-Instruct-2507}~\citep{Qwen3} as the base policy for all trained methods. As training-free references, we evaluate the frozen base policy with a ReAct-style prompt~\citep{ReAct}, together with a suite of strong closed-source models under the same prompt, namely \textit{Gemini-2.5-Flash}~\citep{Gemini2.5}, \textit{Gemini-2.5-Pro}~\citep{Gemini2.5}, \textit{Claude-Sonnet-4.5}~\citep{ClaudeSonnet45}, \textit{Claude-Opus-4.5}~\citep{ClaudeOpus45}, \textit{Claude-Sonnet-4.6}~\citep{ClaudeSonnet46}, and \textit{Claude-Opus-4.6}~\citep{ClaudeOpus46}. For trained baselines, we compare against outcome-only RLVR methods that fine-tune the same base policy with the same sparse terminal reward: GRPO~\citep{GRPO}, GSPO~\citep{GSPO}, and DAPO~\citep{DAPO}. This keeps the terminal reward fixed and focuses the comparison on the added solver-derived turn-level signal. We also include GiGPO~\citep{GiGPO}, a process-level baseline, to examine whether solver-guided supervision provides gains beyond anchor-state process shaping.

\paragraph{Implementation Details}

The agent plays each game through a multi-turn ReAct-style loop, alternating between reasoning over a textual board rendering and emitting one action per turn, until the episode ends on a solve, an illegal move, or an exhausted turn budget (\appref{app:agent_impl}). Sokoban, Minesweeper, and Rush Hour share this interface but differ in their state encodings and step dynamics, which we detail in \appref{app:env_impl}. For each game we build a dedicated solver that returns, for any reachable state, the optimal action value $Q^{\pi_{\text{Solver}}}$ and state value $V^{\pi_{\text{Solver}}}$: weighted A$^{\star}$ search for Sokoban, constraint-satisfaction solving with exact combinatorial mine-probability inference for Minesweeper, and a precomputed full-state distance table via multi-source reverse BFS for Rush Hour (\appref{app:solver_impl}). We query these solvers on the states visited during rollouts to form the solver-advantage signal.

\method is implemented on top of the DAPO backbone, from which we retain the token-mean loss and clip-higher objective, augmenting its outcome objective with the solver-advantage process signal described in \autoref{sec:method}. Following common practice, all environments use a 0/1 sparse reward. We train with a batch size of 16, 8 rollouts per prompt, and a maximum response length of 16384 tokens; the solver-advantage weight is set to $\alpha{=}0.1$ by default, and the remaining training settings are listed in \appref{app:hyperparams}. For evaluation, we score every game by success rate. To reduce sampling variance, for each instance we draw four independent rollouts and report the average success rate, denoted Avg@4. Unless otherwise noted, all tables report the mean over \textbf{3 independent runs}, taking each run's final checkpoint at step 200 (Sokoban), 400 (Minesweeper), and 300 (Rush Hour) for both within-game and OOD-transfer evaluation; closed-source references use a single run under the same Avg@4 protocol.

\subsection{Main Results on the Training Games}
\label{sec:exp_main}
\definecolor{ourshl}{RGB}{222,235,247}

\begin{table*}[t]
    \centering
    \small
    \setlength{\tabcolsep}{5.2pt}
    \renewcommand{\arraystretch}{1.1}
    \caption{\textbf{Main results on the training games.} Avg@4 success rate (\%) on in-domain (ID) and unseen-difficulty (Unseen) levels; \textbf{Average} aggregates the three games. Within the Qwen3-4B-Instruct-2507 block, bold and underlined values mark the best and second-best results, respectively.}
    \label{tab:main_results}
    \begin{tabular}{ll cc cc cc cc}
        \toprule
        \multirow{2}{*}{\textbf{Type}} & \multirow{2}{*}{\textbf{Approach}}
            & \multicolumn{2}{c}{\textbf{Sokoban}}
            & \multicolumn{2}{c}{\textbf{Minesweeper}}
            & \multicolumn{2}{c}{\textbf{Rush Hour}}
            & \multicolumn{2}{c}{\textbf{Average}} \\
        \cmidrule(lr){3-4} \cmidrule(lr){5-6} \cmidrule(lr){7-8} \cmidrule(lr){9-10}
            & & \textbf{ID} & \textbf{Unseen} & \textbf{ID} & \textbf{Unseen} & \textbf{ID} & \textbf{Unseen} & \textbf{ID} & \textbf{Unseen} \\
        \midrule
        \multicolumn{10}{l}{\textit{Closed-Source Models (ReAct; G = Gemini)}} \\
        Prompting & G-2.5-Flash & 64.0 & 33.3 & 38.1 & 23.0 & 73.9 & 63.6 & 58.7 & 40.0 \\
        Prompting & G-2.5-Pro & 77.5 & 45.6 & 50.8 & 18.8 & 92.0 & 88.9 & 73.4 & 51.1 \\
        Prompting & Sonnet-4.5 & 60.0 & 29.3 & 34.4 & 22.8 & 56.9 & 42.3 & 50.4 & 31.5 \\
        Prompting & Opus-4.5 & 71.9 & 37.3 & 47.3 & 29.4 & 70.4 & 59.1 & 63.2 & 41.9 \\
        Prompting & Sonnet-4.6 & 79.5 & 47.1 & 57.6 & 42.5 & 88.1 & 81.4 & 75.1 & 57.0 \\
        Prompting & Opus-4.6 & 82.5 & 51.5 & 64.6 & 55.3 & 92.6 & 85.3 & 79.9 & 64.0 \\
        \midrule
        \multicolumn{10}{l}{\textit{Base Model: Qwen3-4B-Instruct-2507}} \\
        Prompting   & ReAct & 44.3 & 16.8 & 4.5 & 0.4 & 1.1 & 0.5 & 16.6 & 5.9 \\
        RL Training & GRPO & 71.3 & 30.5 & 9.7 & 0.5 & \underline{53.8} & 26.7 & 44.9 & 19.2 \\
        RL Training & GSPO & \underline{72.6} & 28.3 & 28.2 & 5.0 & 24.8 & 11.5 & 41.9 & 14.9 \\
        RL Training & DAPO & 71.9 & \underline{31.3} & \underline{29.8} & \underline{7.3} & 32.5 & 17.4 & 44.7 & 18.7 \\
        RL Training & GiGPO & 70.6 & 29.6 & 13.7 & 1.4 & 51.9 & \underline{31.3} & \underline{45.4} & \underline{20.8} \\
        \rowcolor{ourshl}
        RL Training & CAST (Ours) & \textbf{77.0} & \textbf{34.8} & \textbf{44.7} & \textbf{11.0} & \textbf{64.7} & \textbf{39.5} & \textbf{62.1} & \textbf{28.4} \\
        \bottomrule
    \end{tabular}
\end{table*}

\paragraph{Comparison with training-free baselines.}
Training substantially improves the \textit{Qwen3-4B-Instruct-2507} base policy. With a ReAct framework, the frozen base reaches game-averaged Avg@4 scores of 16.6 on ID and 5.9 on Unseen-Difficulty; after training, our 4B agent reaches 62.1 and 28.4. On ID, it also exceeds the ReAct scores of two closed-source models, Gemini-2.5-Flash (58.7) and Claude-Sonnet-4.5 (50.4), showing the practical benefit of solver-guided RL beyond prompting alone for this base policy.

\paragraph{Controlled comparison against trained baselines.}
Trained from the same base with the same final reward, \method attains the best Avg@4 in every game under both evaluation settings. The cleanest comparison is against DAPO, our backbone with the solver-advantage signal removed: adding this signal lifts the ID average from 44.7 to 62.1 and the Unseen-Difficulty average from 18.7 to 28.4, isolating its contribution. The lead also holds over the outcome-only baseline GRPO (44.9 ID, 19.2 Unseen) and the process-level baseline GiGPO (45.4, 20.8), and spans all three games rather than coming from a single environment. Minesweeper shows the largest ID gain of $+14.9$ over the best trained baseline, while its Unseen-Difficulty score of 11.0 leaves room for improvement on harder instances.

\paragraph{Training dynamics.}

\autoref{fig:main_curves} shows that solver guidance improves both training efficiency and validation performance. \method attains the highest validation Avg@4 on all three games and reaches DAPO's peak after only 120, 200, and 140 training steps on Sokoban, Minesweeper, and Rush Hour, respectively, compared with 200, 400, and 240 steps for DAPO, corresponding to a $1.7$--$2.0\times$ speedup. \method leads on training reward across all three games, clearly on the harder Minesweeper and Rush Hour and edging ahead only late on the easier Sokoban.

\begin{figure*}[t]
    \centering
    \includegraphics[width=\textwidth]{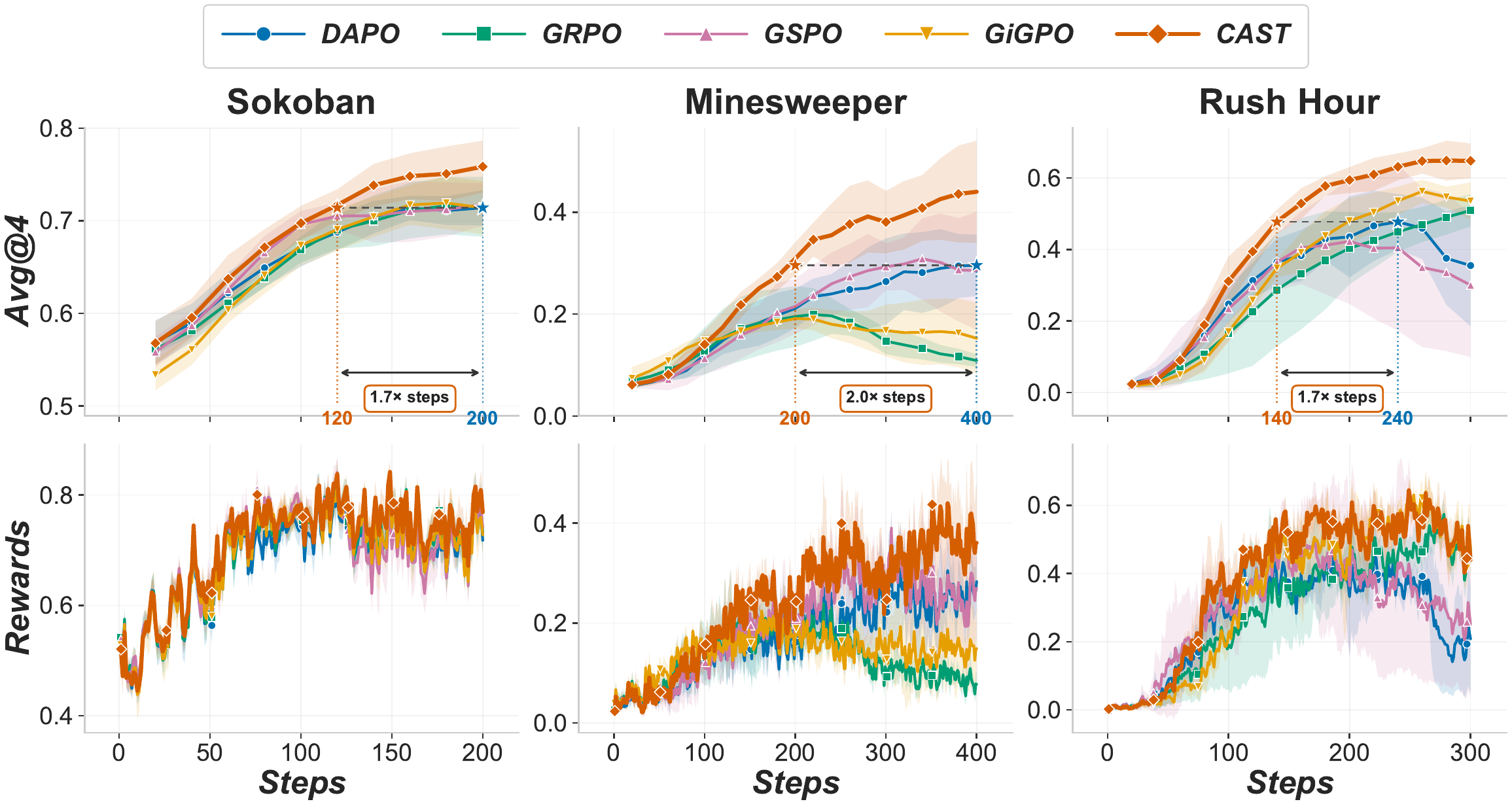}
    \caption{\textbf{Training dynamics.} Horizontal dashed lines mark DAPO's peak validation Avg@4; vertical dotted lines mark when \method and DAPO first reach it, in orange and blue respectively. Curves are EMA-smoothed at $0.6$, with bands showing the standard deviation over three runs.}
    \label{fig:main_curves}
\end{figure*}

\subsection{Zero-Shot OOD Transfer}
\label{sec:exp_ood}
\definecolor{ourshl}{RGB}{222,235,247}

\begin{table*}[t]
    \centering
    \footnotesize
    \setlength{\tabcolsep}{4pt}
    \renewcommand{\arraystretch}{1.1}
    \caption{\textbf{Zero-shot OOD transfer.} Avg@4 success rate (\%) on ALFWorld and WebShop; \textbf{Average} aggregates source-game agents, and \textbf{Overall} aggregates both domains. Dashes denote entries not applicable to ReAct; bold and underlined values mark the best and second-best results, respectively.}
    \label{tab:ood_domain}
    \begin{tabular}{ll cccc cccc c}
        \toprule
        \multirow{2}{*}{\textbf{Type}} & \multirow{2}{*}{\textbf{Method}}
            & \multicolumn{4}{c}{\textbf{ALFWorld}}
            & \multicolumn{4}{c}{\textbf{WebShop}}
            & \multirow{2}{*}{\textbf{Overall}} \\
        \cmidrule(lr){3-6} \cmidrule(lr){7-10}
            & & \textit{Sok.} & \textit{Mine.} & \textit{Rush} & \textbf{Average}
            & \textit{Sok.} & \textit{Mine.} & \textit{Rush} & \textbf{Average} & \\
        \midrule
        \multicolumn{11}{l}{\textit{Base: Qwen3-4B-Instruct-2507}} \\
        Prompting   & ReAct & --- & --- & --- & 30.2 & --- & --- & --- & \underline{18.8} & 24.5 \\
        RL Training & GRPO & 31.5 & 22.3 & 27.9 & 27.2 & \underline{20.4} & 14.0 & 18.0 & 17.5 & 22.4 \\
        RL Training & GSPO & 34.2 & 22.0 & \underline{40.0} & \underline{32.1} & 18.0 & 13.1 & \textbf{20.9} & 17.3 & \underline{24.7} \\
        RL Training & DAPO & \underline{37.3} & \textbf{35.2} & 18.6 & 30.4 & 17.5 & 18.8 & 13.6 & 16.6 & 23.5 \\
        RL Training & GiGPO & 35.6 & 24.9 & 32.1 & 30.9 & 17.2 & \underline{20.2} & 16.4 & 17.9 & 24.4 \\
        \rowcolor{ourshl}
        RL Training & CAST (Ours) & \textbf{38.4} & \underline{33.9} & \textbf{41.4} & \textbf{37.9} & \textbf{23.2} & \textbf{25.1} & \underline{19.7} & \textbf{22.7} & \textbf{30.3} \\
        \bottomrule
    \end{tabular}
\end{table*}

\autoref{tab:ood_domain} shows that the benefits of solver-guided training extend beyond the source games. \method achieves the highest domain average on both ALFWorld (37.9) and WebShop (22.7), exceeding the strongest trained baseline by 5.8 and 4.8 points, respectively. It also outperforms training-free ReAct on both domains and attains the best Overall score of 30.3, 5.6 points above the second-best method. A WebShop case study is provided in \appref{app:webshop_case}.

\subsection{Ablation Studies}
\label{sec:exp_ablation}
We ablate the key design choices of \method on Sokoban.
\begin{enumerate}[noitemsep,topsep=0pt,leftmargin=*]
    \item \textbf{Solver-advantage weight $\alpha$} (\autoref{fig:ablation_curves} left): the weight trades off the process signal against the outcome objective, and our default $\alpha{=}0.1$ is best on both validation Avg@4 and training reward. When $\alpha$ is too small the process signal is diluted and the run behaves like the outcome-only baseline, converging slowly to a lower plateau. When $\alpha$ is too large the signal crowds out the sparse outcome objective: validation Avg@4 climbs quickly at first but peaks early and then declines, and training reward becomes unstable and drops in late training. Only $\alpha{=}0.1$ keeps improving and stays stable throughout.
    \item \textbf{Transformation and normalization} (\autoref{fig:ablation_curves} right): the \texttt{asinh} transformation and batch-level RMS normalization play complementary roles, and removing either lowers the final Avg@4. Dropping \texttt{asinh} hurts most: without it the run is slowest and stays lowest throughout, since unbounded solver advantages let a few large values dominate the update. Dropping RMS normalization instead matches the full method early but plateaus and dips late, with noisier training reward, indicating that per-batch rescaling is what keeps the signal stable as training progresses.
\end{enumerate}

\begin{figure*}[t]
    \centering
    \begin{subfigure}{0.48\textwidth}
        \centering
        \includegraphics[width=\linewidth]{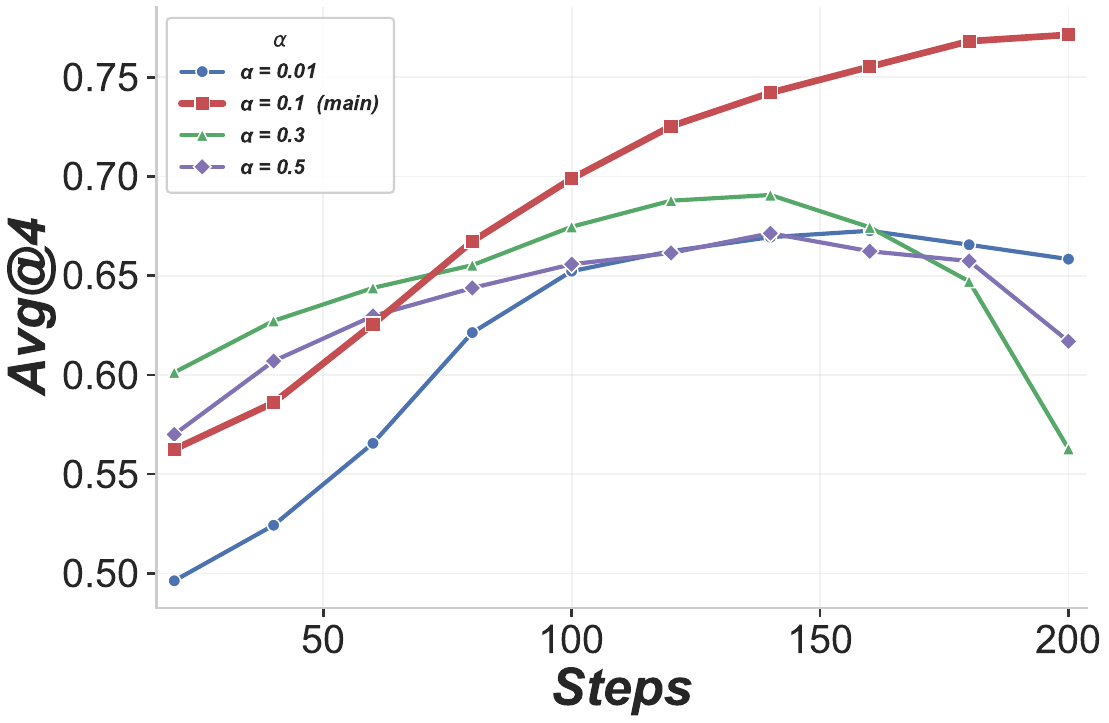}
        \caption{Solver-advantage weight $\alpha$: Val Avg@4}
        \label{fig:ablation_alpha_pass}
    \end{subfigure}
    \hfill
    \begin{subfigure}{0.48\textwidth}
        \centering
        \includegraphics[width=\linewidth]{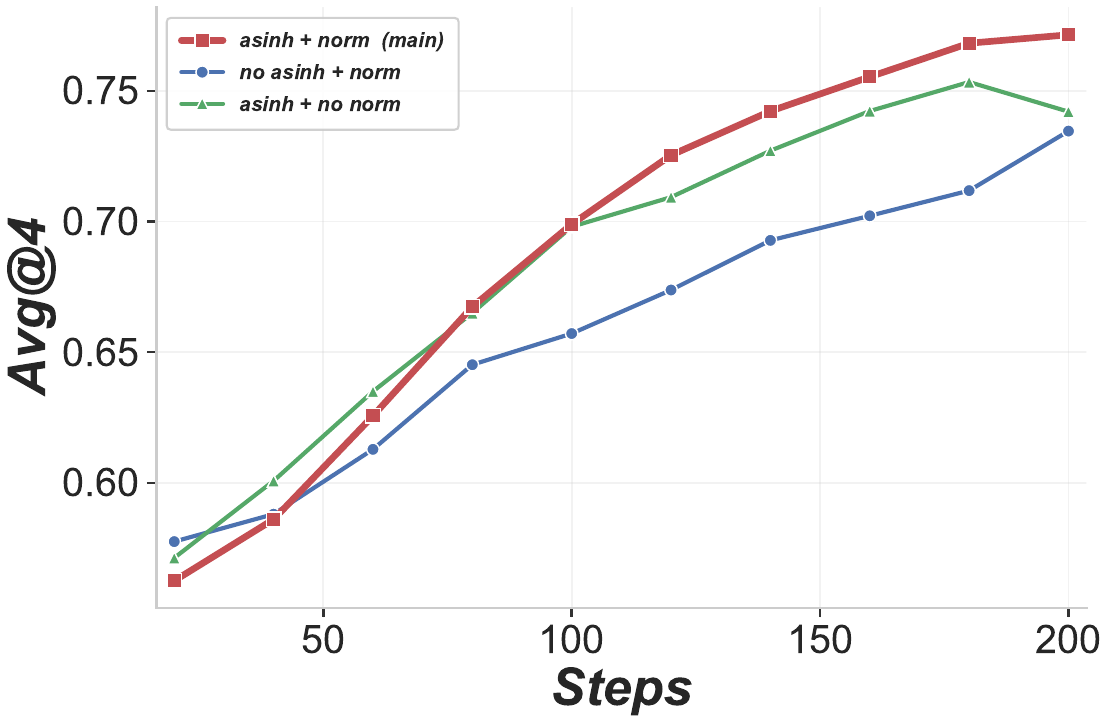}
        \caption{\texttt{asinh} + batch RMS: Val Avg@4}
        \label{fig:ablation_asinh_pass}
    \end{subfigure}

    \vspace{4pt}

    \begin{subfigure}{0.48\textwidth}
        \centering
        \includegraphics[width=\linewidth]{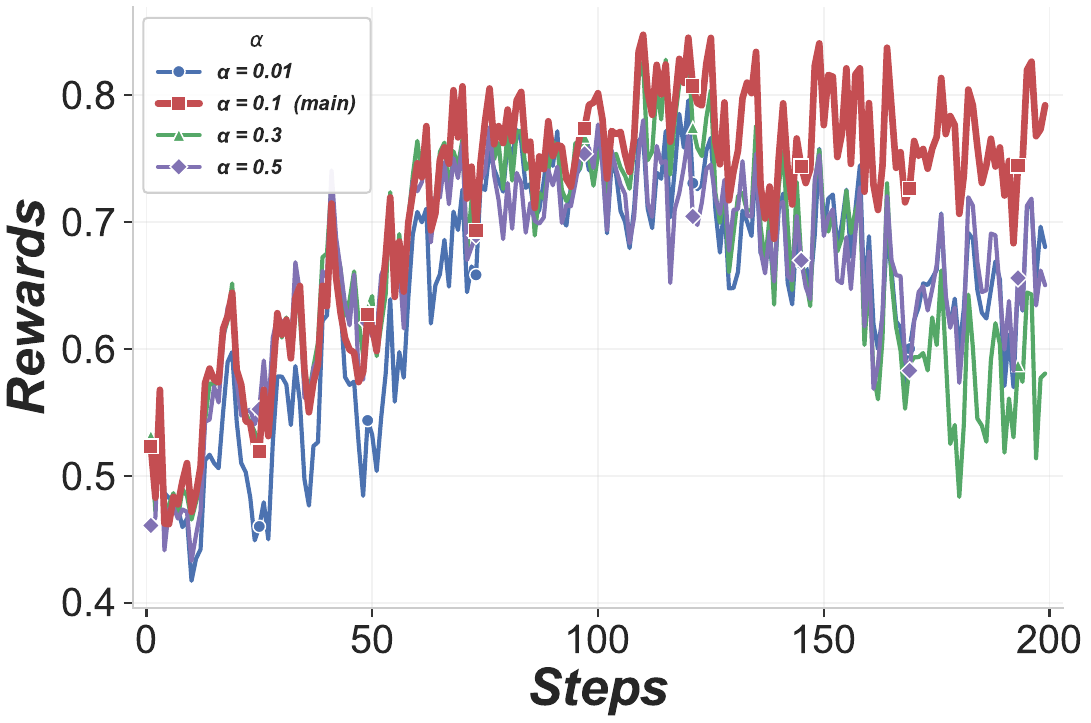}
        \caption{Solver-advantage weight $\alpha$: Train rewards}
        \label{fig:ablation_alpha_reward}
    \end{subfigure}
    \hfill
    \begin{subfigure}{0.48\textwidth}
        \centering
        \includegraphics[width=\linewidth]{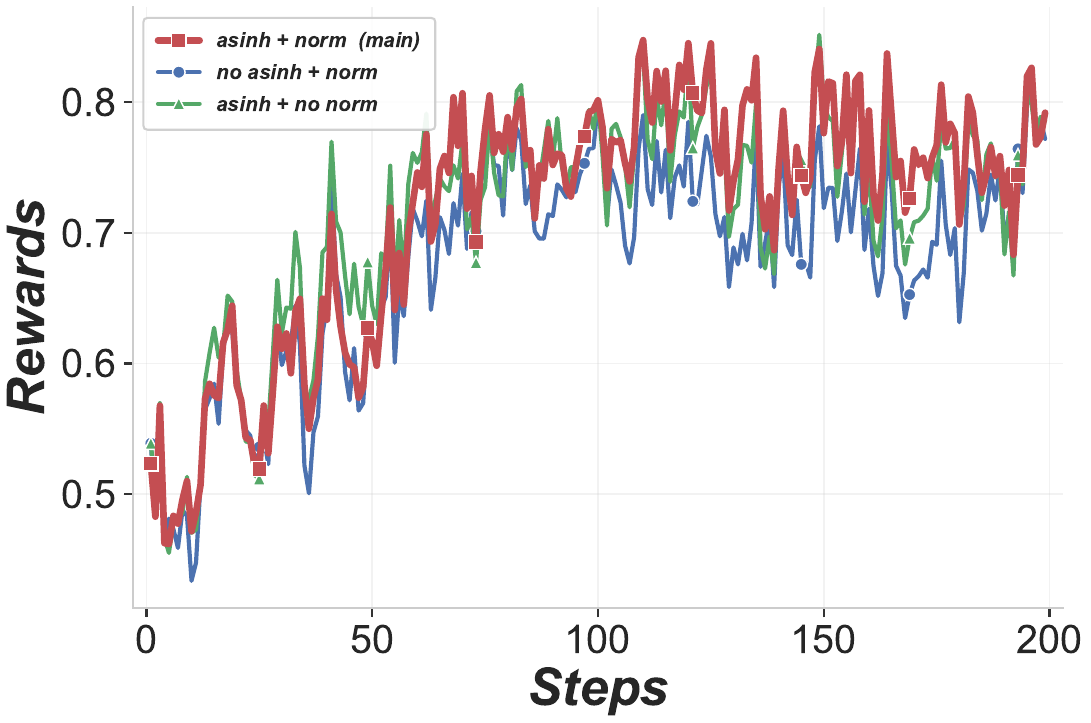}
        \caption{\texttt{asinh} + batch RMS: Train rewards}
        \label{fig:ablation_asinh_reward}
    \end{subfigure}

    \caption{\textbf{Ablation studies on Sokoban.} \textbf{Left column}: sweeping the solver-advantage weight $\alpha$. \textbf{Right column}: removing/replacing the \texttt{asinh} transformation and batch-level RMS normalization. Curves are EMA-smoothed at $0.6$.}
    \label{fig:ablation_curves}
\end{figure*}

\subsection{Analysis}
\label{sec:exp_analysis}
We analyze the practicality of solver guidance from two angles: the training-time overhead it adds, and whether it still helps when the exact solver is replaced by an approximate learned value network.

\paragraph{Solver runtime overhead.}
Forming the solver-advantage signal requires one solver query per environment step to obtain $Q^{\pi_{\text{Solver}}}$ and $V^{\pi_{\text{Solver}}}$ for the current state. \autoref{fig:solver_time} shows that the relative cost of these queries rapidly diminishes at broader levels of the training pipeline. Although solver queries account for $8.4\%$ of an environment step, environment interaction itself occupies only $0.1\%$ of trajectory runtime, with the remaining $99.9\%$ dominated by LLM generation. Solver time therefore accounts for only $0.01\%$ of a trajectory. Since rollout collection occupies $60.7\%$ of a training step, the solver's estimated end-to-end contribution is just $73$ ppm of total training-step wall-clock time, making its overhead negligible relative to generation and policy optimization.

\paragraph{Learned value networks as solvers.}
Solver guidance can also use an approximate learned value function as the process-signal source, which extends it to the traditional RL or deep RL. On Rush Hour, we replace the exact solver with a DQN-based value network~\citep{DDQN} trained without solver distances, then refined by self-distillation and a small geometric prior; we use its state values as the process signal (details in \appref{app:dqn}). As shown in \autoref{fig:analysis_dqn}, the learned-network variant closely tracks the exact-solver version and ends only slightly below it in validation Avg@4. It remains above the trained baselines at the end of training and, unlike DAPO and GSPO, shows no comparable late-stage drop in validation Avg@4. Thus, even with a learned value function, \method retains much of the benefit of exact-solver guidance, showing that solver guidance carries over to standard RL rather than depending on a hand-built solver.

\begin{figure}[t]
    \centering
    \begin{subfigure}{0.49\linewidth}
        \centering
        \includegraphics[width=\linewidth]{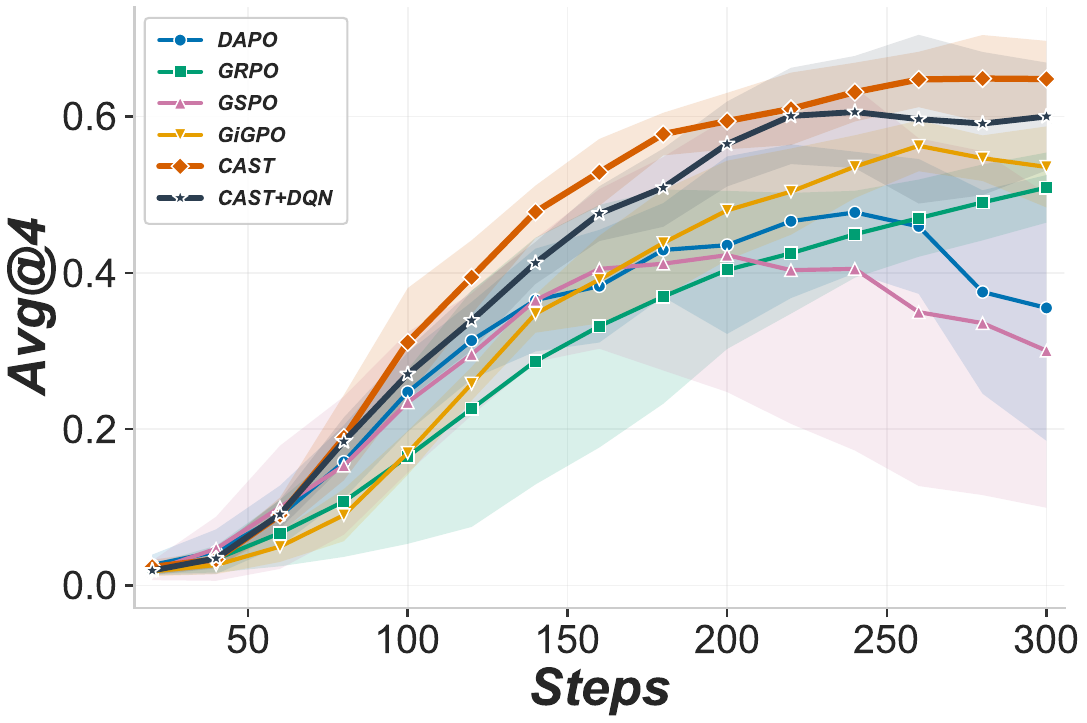}
        \caption{Val Avg@4}
        \label{fig:dqn_pass}
    \end{subfigure}
    \hfill
    \begin{subfigure}{0.49\linewidth}
        \centering
        \includegraphics[width=\linewidth]{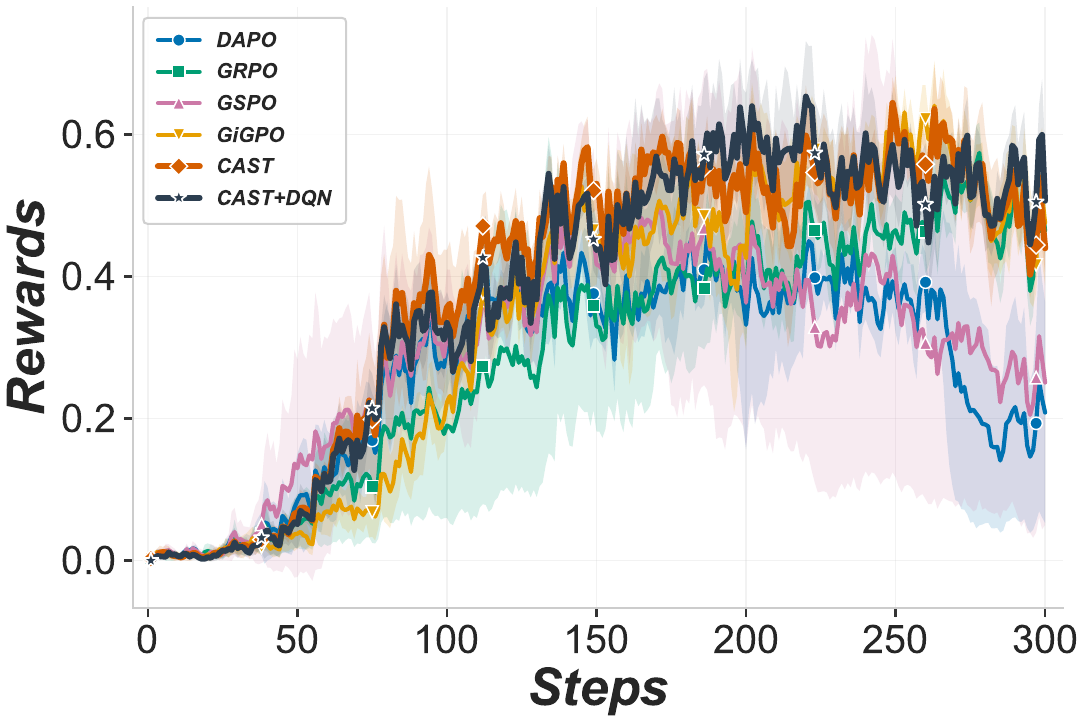}
        \caption{Train rewards}
        \label{fig:dqn_reward}
    \end{subfigure}
    \caption{\textbf{Learned value network as a solver on Rush Hour.} We replace the exact solver with a DQN-based value network trained without solver distances as the process-signal source. Curves are EMA-smoothed at $0.6$; bands show the standard deviation across the 3 runs.}
    \label{fig:analysis_dqn}
\end{figure}

\section{Conclusion}

We address turn-level credit assignment in RLVR for LLM game agents, where sparse terminal rewards provide little direct guidance about which actions lead to success or failure. \method queries a game solver on LLM-visited states and converts per-action changes in the solver's state value into solver advantages that provide fine-grained, turn-level credit; under a soft-optimal solver assumption, this update also admits a logit-free on-policy distillation interpretation without teacher logits. Across Sokoban, Minesweeper, and Rush Hour, \method achieves the best performance among trained methods on every game under both in-domain and unseen-difficulty evaluation, reaches DAPO's peak validation performance in $1.7$--$2.0\times$ fewer training steps, and attains the highest average zero-shot performance on held-out ALFWorld and WebShop. Solver queries add negligible training overhead, while an approximate learned value network retains much of the benefit of exact solver guidance. More broadly, our results suggest that reliable state evaluation offers a practical route from sparse outcomes to fine-grained credit, whether the evaluator is an exact solver or a learned value function.

\clearpage

\bibliography{example_paper}
\bibliographystyle{iclr2026_conference}

\clearpage

\appendix

\section*{Appendix}

\section{Related Work}

\paragraph{LLMs as Game Agents.}
Recent work has moved foundation models from passive text generation toward interactive game play. Prompted and scaffolded agents equip LLMs or VLMs with memory, planning, reflection, and coordination, showing that they can interact with games in open world exploration~\citep{Voyager}, general computer and video game control~\citep{Cradle}, and reflective or multi-agent settings~\citep{Reflexion}, often without updating model weights. In parallel, game environments and benchmarks make such interaction measurable and scalable: open world platforms~\citep{Minedojo} and recent benchmarks for LLMs and VLMs~\citep{BALROG, lmgame-Bench, Textarena, GameWorld} provide testbeds for evaluating and training game agents. Game RL with outcome supervision then turns final rewards from games and puzzles into training signals for improving reasoning, transfer, or performance in specific games~\citep{ViGaL, G1, Enigmata, Spiral}. These signals are clean and verifiable, but remain coarse for long trajectories: they indicate whether an attempt succeeds, while providing limited guidance about which intermediate decisions mattered.

\paragraph{Credit Assignment and Process Signals in Agentic RL.}
To address the coarse credit provided by final outcomes, recent agentic RL methods construct process signals below the trajectory level, at the segment, turn, or token~\citep{Turn-ppo, SPO, Spa-rl}. Some methods redistribute outcome information through grouping or decomposition, combining trajectory-level advantages with step-level comparisons over recurring or comparable states~\citep{GiGPO, HGPO}. Others learn process evaluators for agent trajectories, including process reward models that estimate stepwise promise and progress~\citep{Agentprm}, and learned turn-level critics for multi-turn environments~\citep{Vagen}. Search- and rollout-based approaches instead estimate intermediate values by expanding possible continuations, producing value estimates or process labels for policy improvement~\citep{RAP, Vineppo, ProAct}. These signal sources expose a recurring trade-off: grouping methods require comparable intermediate states, learned verifiers and critics require additional training signals or models, and search or rollout estimates can add substantial computation or depend on the rollout policy. In contrast, our work uses game-specific solvers as an external source of turn-level supervision.

\paragraph{Logit-Free Distillation from Solvers and Search.}
\method is also related to distillation, but differs in the teacher signal it uses. Recent LLM on-policy distillation learns from student-generated outputs, but assumes access to teacher token distributions or log probabilities~\citep{GKD, Minillm, OPD}. Classical solvers instead provide structured signals such as actions, feasibility, and cost-to-go values, connecting our setting to traditional imitation learning and search distillation. DAgger queries an oracle on learner-visited states~\citep{DAgger}, AggreVaTe supervises actions with cost-to-go estimates~\citep{AggreVaTe}, and expert-iteration methods distill search-improved decisions into a policy~\citep{ExIt, AlphaZero}. These methods target task-specific policies, whereas we attach scalar solver advantage estimates to LLM-sampled actions, yielding a logit-free process signal that plugs directly into RLVR.

\section{Datasets and Benchmarks}
\label{app:data_and_benchmarks}

\subsection{Dataset Details}
\label{app:dataset_details}

\paragraph{Data generation.}
Each game uses a procedural generator that produces instances at a controllable difficulty. \textbf{Sokoban} rooms are grown by random wall placement and then populated by a reverse-playing search, which guarantees a solvable layout in which at least one box must be pushed. \textbf{Minesweeper} boards place mines uniformly at random outside a $3{\times}3$ safe zone around a designated first-click cell, which the environment reveals at the start of every episode so that the board is fixed from the agent's first action onward; we do not filter for no-guess solvability, so a few boards may still require a probabilistic guess, which the oracle handles as described in Appendix~\ref{app:solver_impl}. \textbf{Rush Hour} boards are sampled at random and kept only if an exact solver confirms an optimal solution whose length lies in a target range; this solver is used only for data construction and differs from the training oracle of Appendix~\ref{app:solver_impl}.

\paragraph{Difficulty and splits.}
Difficulty is controlled by board size and a game-specific count: boxes for Sokoban, mines for Minesweeper, and total vehicles for Rush Hour. For each game, we train on a single in-domain (ID) tier and hold out a harder unseen-difficulty tier, obtained by increasing this count and, for Sokoban and Minesweeper, enlarging the board; both tiers are listed in Table~\ref{tab:dataset_difficulty}. Each tier provides $200$ evaluation instances. All instances are drawn from independent random seeds and de-duplicated by a hash of their initial board, so that no instance recurs within the training set and no evaluation instance appears in training. The textual state encoding and action interface of each environment are described in Appendix~\ref{app:env_impl}.

\begin{table}[h]
\centering
\small
\caption{Per-game in-domain (ID) and unseen-difficulty (Unseen) tiers. Agents are trained only on ID and evaluated on both tiers; for Rush Hour, the count includes the target vehicle \texttt{A}.}
\label{tab:dataset_difficulty}
\begin{tabular}{llcc}
\toprule
\textbf{Game} & \textbf{Split} & \textbf{Board} & \shortstack{\textbf{\# Boxes, Mines,} \\ \textbf{or Vehicles}} \\
\midrule
\multirow{2}{*}{Sokoban}
    & ID   & $6{\times}6$ & 2 boxes \\
    & Unseen & $7{\times}7$ & 3 boxes \\
\midrule
\multirow{2}{*}{Minesweeper}
    & ID   & $6{\times}6$ & 7 mines \\
    & Unseen & $7{\times}7$ & 10 mines \\
\midrule
\multirow{2}{*}{Rush Hour}
    & ID   & $6{\times}6$ & 7 vehicles \\
    & Unseen & $6{\times}6$ & 9 vehicles \\
\bottomrule
\end{tabular}
\end{table}

\subsection{OOD Transfer Benchmark Details}
\label{app:benchmark_details}

To measure cross-domain transfer, we take each agent trained on the three games and evaluate it zero-shot, with no further training, on two standard text-agent benchmarks, ALFWorld and WebShop. Both are driven by the same multi-turn loop and single-action interface as the games (Appendix~\ref{app:agent_impl}), differing only in their observations and action space, so that any success reflects abilities transferred from game training rather than benchmark-specific tuning.

\paragraph{ALFWorld.}
ALFWorld~\citep{Alfworld} is an embodied household benchmark in which the agent completes a natural-language instruction (e.g.\ heating an object and placing it) by navigating and manipulating objects in a text-rendered room. We use its TextWorld engine over the six standard task types and evaluate on the \emph{seen} validation split. Each turn presents the current observation together with the list of admissible commands, and the agent replies with one command such as \texttt{go to}, \texttt{take}, \texttt{open}, or \texttt{heat}; an invalid command is rejected with feedback and the episode continues. An episode runs for up to $30$ turns.

\paragraph{WebShop.}
WebShop~\citep{Webshop} is a simulated online store in which the agent must fulfill a shopping instruction by searching, browsing, and purchasing a product that matches the requested attributes and price. The text observation renders the current page, and the agent acts through \texttt{search[query]} and \texttt{click[button]}, choosing from the clickable elements listed for that page. We evaluate on the standard $500$ test instructions, with each episode capped at $30$ turns. WebShop reports two metrics, a continuous match score and a binary success rate; we use the success rate, counting an instruction as solved only on a perfect match.

Both benchmarks are decoded with the same open-source configuration as the games (Table~\ref{tab:eval_open}): SGLang serving, temperature $0.6$, top-$p$ $0.95$, up to $16{,}384$ tokens per turn and $2{,}048$ prompt tokens, and $4$ rollouts per instance.

\subsection{OOD Transfer Case Study}
\label{app:webshop_case}

We compare the base model \textit{Qwen3-4B-Instruct-2507} and its Sokoban-trained model on the same WebShop instruction; the latter receives only solver-guided Sokoban training before zero-shot evaluation. The instruction asks, \say{\emph{Find me easy clean sofa tables with solid wood for living room with color: walnut, and price lower than 260.00 dollars}}. WebShop credits a purchase only on a full attribute match; here, the \texttt{walnut} option must be selected on the product page before the irreversible \texttt{buy now}. The base model commits to an invalid brown table, whereas the \method-trained model rejects a poor candidate, searches again, and selects the required color before buying. This contrast illustrates recovery from an intermediate mistake and constraint checking before commitment, behaviors consistent with the planning discipline learned in Sokoban. Every action is retained, while some reasoning is omitted for brevity and marked as \say{[\dots]}.

\begin{cstraj}{csfail}{\faTimesCircle~~Base model (untrained) \hfill 3 steps \; -- \; reward 0.0 \; \faTimesCircle}
\csstep{0}\csobs{search page.}
\csthink{A good search term would be `solid wood sofa table' as it captures the essential product type and material. This avoids over-specificity and allows the system to return relevant results that can be filtered later for color and price.}
\csact{search[solid wood sofa table]}

\csstep{1}\csobs{50 results; none of the listed titles says \emph{walnut}.}
\csthink{[\dots] The most suitable product that meets the solid wood and price criteria is the `Wyndenhall' table, even though it's brown. It is solid wood, priced under \$260, and fits the purpose.}
\csact{click[b084hbqytd]}

\csstep[\textcolor{cspivotc}{\faStar}]{2}\csobs{Wyndenhall product page (\say{Natural Aged Brown}); it exposes \emph{no} color option.}
\csthink{[\dots] We cannot do a second one with `walnut' because that would include color --- which violates the rule. Therefore, we cannot search again. [\dots] we should proceed with buying the best available match.}
\csnote{csfail}{over-applies a prompt tip, refuses to explore, and commits the irreversible \texttt{buy} on a page that can never satisfy the goal.}
\csact{click[buy now]\hfill\textcolor{csfail}{\bfseries wrong item -- reward 0.0}}
\end{cstraj}

\begin{cstraj}{csok}{\faCheckCircle~~\method (Sokoban-trained, zero-shot) \hfill 6 steps \; -- \; reward 1.0 \; \faCheckCircle}
\csstep{0}\csobs{search page.}
\csthink{Since I cannot include size or color in the search term (as per tips), I should search for `sofa tables solid wood' to get relevant products. [\dots]}
\csact{search[sofa tables solid wood]}

\csstep{1}\csobs{50 results; none of the titles says \emph{walnut}.}
\csthink{The best match in terms of solid wood and low price is the `SunnyPoint Classic Side Table' (B084ZPN17C), even though it's not walnut. [\dots] I will click on the SunnyPoint table, assuming it might be a close enough match for the task.}
\csact{click[b084zpn17c]}

\csstep[\textcolor{cspivotc}{\faStar}]{2}\csobs{SunnyPoint side table page --- black, metal frame, not solid walnut wood.}
\csthink{The current product does not meet the color requirement. I should not proceed with this one. I need to go back to search with a refined term that includes `walnut' to get better results. [\dots] I will now search for `sofa tables walnut solid wood' to target the exact color and material.}
\csnote{csok}{rejects a bad candidate and spends its second search to recover, rather than settling.}
\csact{search[sofa tables walnut solid wood]}

\csstep{3}\csobs{re-ranked results; the Baja solid-wood table (\$237) now appears near the top.}
\csthink{The Baja Night Stand (B079N3VLRJ) is solid wood, priced under \$260, and is suitable for living room use. [\dots] I will proceed to click on it as the best available option within constraints.}
\csact{click[b079n3vlrj]}

\csstep[\textcolor{cspivotc}{\faStar}]{4}\csobs{Baja product page exposes a color selector: \texttt{brown / grey / \textbf{walnut} / white}.}
\csthink{The available color options in the page show: brown, grey, walnut, white. I should click on the `walnut' color option to see if any product with walnut finish appears. [\dots] This is the best next step to fulfill the color requirement.}
\csnote{csok}{sees a \texttt{walnut} option and selects it \emph{before} buying --- the move that actually registers the required attribute.}
\csact{click[walnut]}

\csstep{5}\csobs{page looks unchanged after selecting \texttt{walnut}.}
\csthink{Clicking on `walnut' did not change the product listing [\dots] I should buy this product as the best available option. Therefore, I will proceed with `Buy Now' for this product.}
\csact{click[buy now]\hfill\textcolor{csok}{\bfseries perfect match -- reward 1.0}}
\end{cstraj}

\section{Implementation Details and Hyperparameters}
\label{app:detailed_implementation}

\subsection{Agent Implementation}
\label{app:agent_impl}

All three games share a single agent scaffold, with game-specific system prompts, observation formatting, action parsers, and feedback. At each turn the agent appends the environment's text observation of the current board to the context and asks the policy to reason and then commit to one action, which it parses and passes back to the environment; the loop repeats until the game is solved, fails, or reaches its turn budget.

\paragraph{Observation and action format.}
Each turn shows the current board together with the number of remaining steps, and, when the previous action was invalid or ineffective, a short line of corrective feedback. Sokoban and Rush Hour render the board as a symbolic grid paired with a coordinate listing of each entity, while Minesweeper shows the symbolic grid with row and column index headers but no separate coordinate listing. The policy is asked to reason and then place a single action inside a fenced block, such as \verb|```Up```| for Sokoban, \verb|```reveal 3 2```| for Minesweeper, or \verb|```A+2```| for Rush Hour. We read the last fenced block in the response, so intermediate reasoning is ignored, and parse it with a per-game rule into a concrete action.

\paragraph{History and trajectories.}
The agent conditions on the full context history, including its own reasoning, so that turn $t$ sees every earlier observation and response. The resulting trajectory, the ordered sequence of observations, responses, actions, and the terminal $0/1$ reward, is the token-level rollout optimized by the RL objectives in Appendix~\ref{app:hyperparams_train}.

\paragraph{Invalid actions and termination.}
An unparseable or illegal action, such as an unknown token, an out-of-bounds cell, or a move blocked by a wall, does not end the episode: it leaves the state unchanged, consumes one turn, and returns corrective feedback so the policy can retry. An episode ends only when the game is solved, a game-specific failure occurs (a mine hit in Minesweeper, or a deadlock in Sokoban), or the turn budget is exhausted, and yields reward $1$ only on success.

\paragraph{System prompts.}
Each environment uses a fixed system prompt that specifies its rules and action format. For reproducibility, we list the exact prompt of every environment below: the three training games first, followed by the two OOD transfer benchmarks.

\promptbox{Sokoban}{prompts/sokoban.txt}
\promptbox{Minesweeper}{prompts/minesweeper.txt}
\promptbox{Rush Hour}{prompts/rush_hour.txt}
\promptbox{ALFWorld}{prompts/alfworld.txt}
\promptbox{WebShop}{prompts/webshop.txt}

\subsection{Environment Implementation}
\label{app:env_impl}

The three environments share the same interface but differ in their state encoding, action semantics, and termination rules. Sokoban and Rush Hour pair a symbolic grid with an explicit zero-indexed coordinate listing, whereas Minesweeper uses a symbolic grid with indexed rows and columns. All environments give a single sparse terminal reward of $1$ on success and $0$ otherwise.

\paragraph{Sokoban.}
The board uses the symbols \texttt{\#} (wall), \texttt{\_} (empty), \texttt{O} (target), \texttt{X} (box), \texttt{*} (box on target), \texttt{P} (player), and \texttt{S} (player on target). The four actions \texttt{Up/Down/Left/Right} move the player, pushing a box one cell when the cell behind it is free; boxes cannot be pulled and walls are impassable. The episode ends in success when every box covers a target, or otherwise on a detected deadlock or the step budget.

\paragraph{Minesweeper.}
Only revealed information is shown: an unrevealed cell is \texttt{.}, a flag is \texttt{F}, a revealed count is a digit \texttt{0}--\texttt{8}, and a mine is \texttt{*} (shown only when hit); a numeric header indexes the rows and columns. The actions are \texttt{reveal R C} and \texttt{flag R C}, where a flag toggles a marker and revealing a $0$ cell flood-fills its zero-valued neighborhood. The episode ends in failure on hitting a mine, in success once every non-mine cell is revealed, or on the step budget. Each instance fixes a latent mine layout and a first-click cell that the environment reveals before the agent acts, so every rollout starts from the same partially revealed board; this removes the initial blind guess that standard Minesweeper would otherwise require.

\paragraph{Rush Hour.}
Vehicles are labeled \texttt{A}, \texttt{B}, \texttt{C}, \dots by index, with the target red car \texttt{A} horizontal on the exit row; empty cells are \texttt{.}, walls are \texttt{x}, and the exit lies at the right edge of \texttt{A}'s row. An action \texttt{$\langle$car$\rangle\langle\pm\rangle\langle$steps$\rangle$} slides one vehicle along its orientation ($+$ is right or down, $-$ is left or up) by a given number of cells, and is all-or-nothing: it executes only if every intermediate cell is free, so a blocked or overshooting move leaves the board unchanged. The episode ends in success when \texttt{A} reaches the exit, or otherwise when the step budget is exhausted.

\subsection{Solver Implementation}
\label{app:solver_impl}

Each game provides a solver-derived cost $N(s)$ that measures the remaining work from state $s$. We define the potential $\Phi(s)=-N(s)$ and per-step oracle advantage $A(s_t,a_t)=\Phi(s_{t+1})-\Phi(s_t)=N(s_t)-N(s_{t+1})$. In the shortest-path domains, this gives $+1$ to an optimal move, $0$ to a wasteful move, and a negative score to harmful moves. Minesweeper uses the analogous solver-completion cost defined below. The solver is queried on states visited along each rollout.

\paragraph{Sokoban.}
We compute $N(s)$ by A$^\ast$ over states $(\textit{player}, \{\textit{boxes}\})$ with a minimum-cost box-to-target assignment as an admissible heuristic. The search returns the shortest solution within its node budget and remains cheap because the training tiers have few boxes on small grids. The same solver also supports the deadlock check used for optional early termination.

\paragraph{Minesweeper.}
Partial observability makes an optimal deterministic cost-to-go ill-defined, so here $K(s)$ measures the work of a deterministic peek-free solver. Using constraint reasoning and global mine probabilities, $K(s)$ is the number of reveal actions needed to clear all safe cells, or $+\infty$ if the solver rollout fails, including when a required probabilistic reveal hits a mine. Because each instance fixes its latent mine layout and first-click reveal (Appendix~\ref{app:env_impl}), $K(s)$ is deterministic for that instance while the agent remains partially observed. The oracle advantage is $K(s_t)-K(s_{t+1})$.

\paragraph{Rush Hour.}
The training oracle uses a different solver from the one that generates the data. Data generation validates each board with a bounded IDA$^\ast$ search (Appendix~\ref{app:dataset_details}). For training, we instead precompute a shortest-path table over the board's reachable component, giving exact $N(s)$ lookups within the configured state cap; larger components fall back to bounded IDA$^\ast$. One slide of any distance counts as a single move.

\begin{figure}[t]
    \centering
    \includegraphics[width=\linewidth]{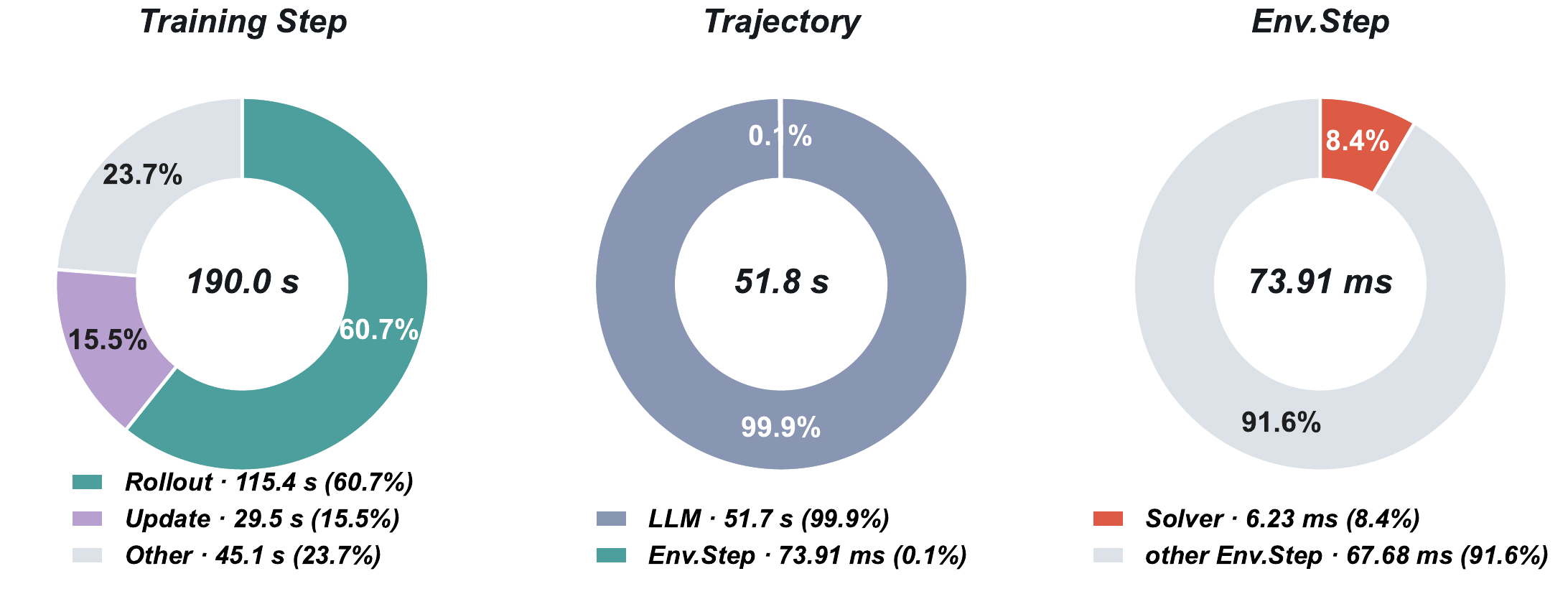}
    \includegraphics[width=0.72\linewidth]{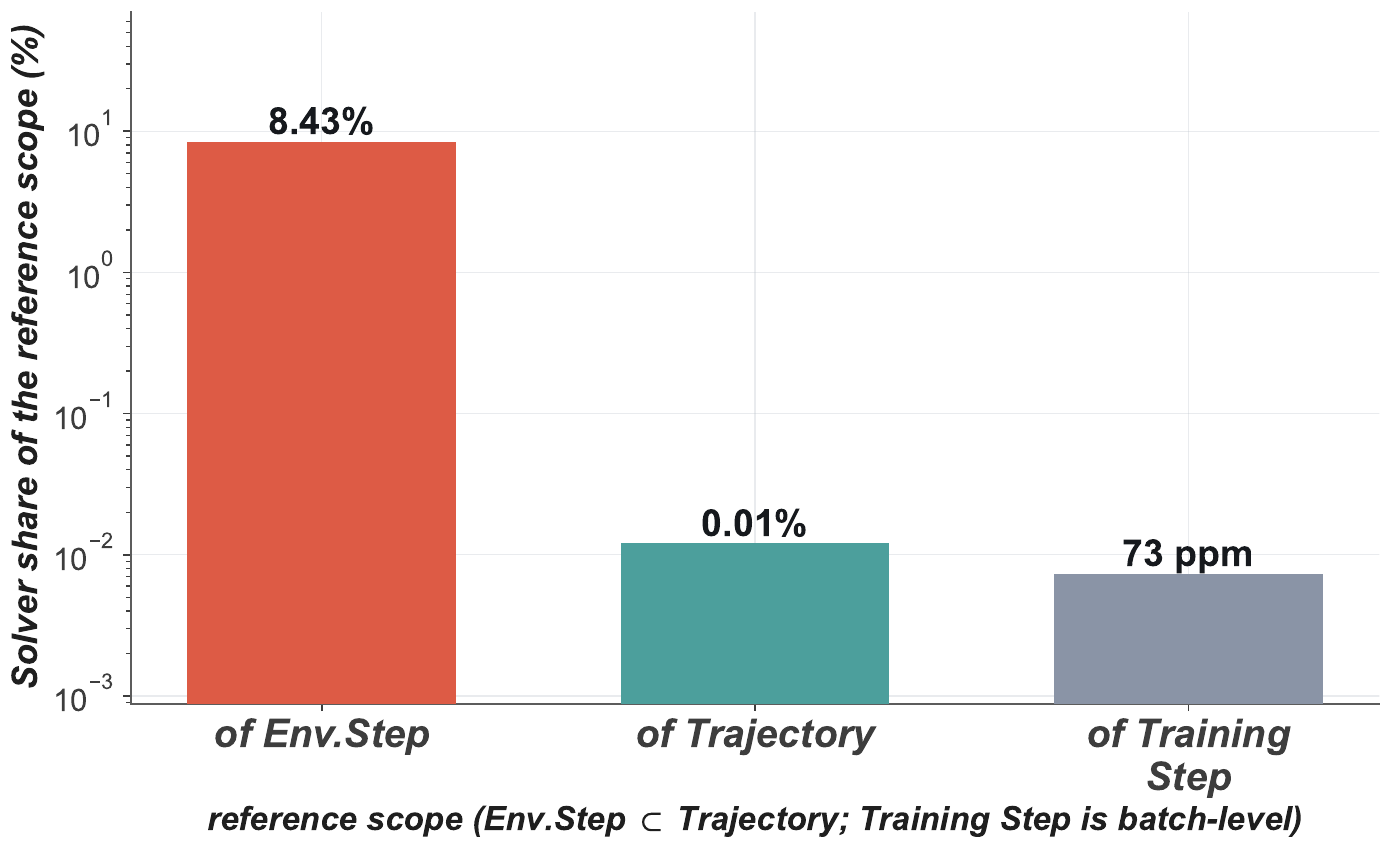}
    \caption{\textbf{Solver overhead on Sokoban.} \textbf{Top}: wall-clock breakdowns at three granularities. The left pie decomposes an average batch-level optimizer step ($190.0$\,s), which includes the parallel collection of $128$ trajectories ($16$ prompts $\times\,8$ rollouts), the policy update, and other operations. Rollout collection, policy update, and other operations account for $60.7\%$, $15.5\%$, and $23.7\%$ of step time, respectively; the last category consists mainly of periodic validation and checkpointing. The middle pie shows that LLM generation accounts for $99.9\%$ of the runtime of a single trajectory, while the right pie shows that solver queries account for $8.4\%$ of a single environment step. \textbf{Bottom}: the solver's time share across increasingly broad scopes on a logarithmic scale. Accounting for parallel rollout collection, its estimated share decreases from $8.43\%$ of an environment step to $0.01\%$ of a trajectory and $73$ ppm of a training step.}
    \label{fig:solver_time}
\end{figure}

\subsection{Learned Value Network as a Solver}
\label{app:dqn}

To test whether an exact solver is required at deployment, we replace the Rush Hour oracle with a DQN-based value network. A masked Double DQN encodes the $6{\times}6$ board with a small transformer and outputs one $Q$-value per action; illegal actions are masked, and $V(s)=\max_{a\in\mathcal{A}(s)}Q(s,a)$. The backend returns the cost surrogate $\widehat{N}(s)=-\text{scale}\cdot V(s)$, so the environment computes $A(s_t,a_t)=\widehat{N}(s_t)-\widehat{N}(s_{t+1})=\text{scale}\,[V(s_{t+1})-V(s_t)]$. Solved states use a fixed value above the non-terminal network range, ensuring that a solving move receives positive advantage.

We train the network in three stages on the same underlying ID split used for LLM training, without generating new puzzle instances. The base network learns from a terminal success reward with a per-step penalty. We then apply two warm-started potential-regression stages: the first distills a one-step lookahead potential, while the second repeats the one-step target and adds the geometric prior $-0.05b(s)$, where $b(s)$ counts occupied cells between the red car and the exit. These optimization targets use game rewards, frozen network predictions, game dynamics, and the geometric prior, but no solver-distance targets. Hyperparameters are listed in Table~\ref{tab:dqn_hp}.

\begin{table}[h]
\centering
\begin{small}
\caption{Training and deployment configuration of the DQN-based value network. All stages use the same architecture and warm-start from the preceding checkpoint.}
\label{tab:dqn_hp}
\begin{tabular}{lc}
\toprule
\textbf{Setting} & \textbf{Value} \\
\midrule
Encoder                 & Transformer, $d_{\text{model}}{=}128$, $8$ heads, $3$ layers, $512$-d output \\
Base RL algorithm       & Masked Double DQN, $n$-step $3$, $\gamma{=}0.99$ \\
Base reward             & Terminal success indicator minus $0.01$ per step \\
Base training steps     & $3{\times}10^{6}$ \\
Distillation objective  & Two stages of one-step potential regression \\
Geometric prior blend   & $-0.05\times$ occupied path-cell count (second stage) \\
Distillation epochs     & $150$ per stage \\
Learning rate           & $10^{-4}$ (base) / $10^{-3}$, $5{\times}10^{-4}$ (distill) \\
Training boards         & $8{,}000$ puzzles from the ID training split \\
Deployment              & $\text{scale}{=}2$, solved-state value ${=}1.0$ \\
\bottomrule
\end{tabular}
\end{small}
\end{table}

\subsection{Equivalence to Cross-Entropy Distillation}
\label{app:kl_proof}

This section formalizes the distillation interpretation of the combined training rule~\autoref{eq:final_advantage}. We prove that, under four explicit assumptions, the implicit objective of the solver-shaped GRPO update is a task return regularized by the cross-entropy between the student and the solver. We then decompose this cross-entropy to reveal both a KL-divergence penalty and a mode-seeking entropy reduction, derive the closed-form optimal policy, characterize the full-domain behavior when the small-signal assumption is relaxed, and give a performance-improvement guarantee via the Performance Difference Lemma.

\subsubsection{Notation}

We work in the finite-horizon MDP $\mathcal{M}=(\mathcal{S},\mathcal{A},P,r,H)$ from \autoref{sec:preliminaries} with undiscounted returns ($\gamma=1$). Let $\pi_\theta$ denote the student (LLM) policy and $\pi_{\text{Solver}}$ the solver's implicit policy. We write $V^{\pi}_{\text{task}},Q^{\pi}_{\text{task}},A^{\pi}_{\text{task}}$ for the value, action-value, and advantage functions under the task reward, and $d^{\pi_\theta}(s)=\sum_{t=0}^{H-1}\Pr(s_t=s\mid\pi_\theta)$ for the (unnormalized) state occupancy measure of $\pi_\theta$, so that $\sum_s d^{\pi_\theta}(s)=H$. We define the following standard information-theoretic quantities for any state $s$:
\begin{align}
\mathcal{H}\big(\pi_\theta(\cdot|s)\big)&=-\textstyle\sum_a\pi_\theta(a|s)\log\pi_\theta(a|s), \label{eq:app_entropy}\\
\mathrm{KL}\big(\pi_\theta(\cdot|s)\,\|\,\pi_{\text{Solver}}(\cdot|s)\big)&=\textstyle\sum_a\pi_\theta(a|s)\log\frac{\pi_\theta(a|s)}{\pi_{\text{Solver}}(a|s)}, \label{eq:app_kl}\\
\mathrm{H}\big(\pi_\theta(\cdot|s),\,\pi_{\text{Solver}}(\cdot|s)\big)&=-\textstyle\sum_a\pi_\theta(a|s)\log\pi_{\text{Solver}}(a|s). \label{eq:app_cross_entropy}
\end{align}
The last quantity is the cross-entropy from $\pi_\theta$ to $\pi_{\text{Solver}}$. These three are related by the identity:
\begin{equation}
\mathrm{H}(\pi_\theta,\,\pi_{\text{Solver}})=\mathrm{KL}(\pi_\theta\,\|\,\pi_{\text{Solver}})+\mathcal{H}(\pi_\theta),
\label{eq:app_ce_decomp}
\end{equation}
which expresses the cross-entropy as the sum of the KL divergence and the entropy of $\pi_\theta$.

Throughout, we use the \emph{baseline invariance} property of policy gradients~\citep{PolicyGradient}: for any function $c(s)$ independent of the action, $\mathbb{E}_{a\sim\pi_\theta}[\nabla_\theta\log\pi_\theta(a|s)\cdot c(s)]=0$.

\subsubsection{Assumptions}
\label{app:assumptions}

\begin{assumption}[Soft-Optimal Solver]\label{ass:soft_optimal}
There exists a temperature $\tau>0$ such that the solver acts as the optimal policy of a soft (maximum-entropy) MDP~\citep{MaxEntIRL,SAC}:
\begin{equation}
\pi_{\text{Solver}}(a|s)=\exp\!\Big(\tfrac{1}{\tau}\big(Q^{\pi_{\text{Solver}}}(s,a)-V^{\pi_{\text{Solver}}}(s)\big)\Big),\qquad
V^{\pi_{\text{Solver}}}(s)=\tau\log\textstyle\sum_{a'}\exp\!\big(Q^{\pi_{\text{Solver}}}(s,a')/\tau\big).
\label{eq:app_soft_optimal}
\end{equation}
\end{assumption}

\begin{assumption}[Small-Signal Regime]\label{ass:small_signal}
The solver advantages encountered during training satisfy $|A^{\pi_{\text{Solver}}}|\lesssim 1$. Since most moves change the cost-to-go by $0$ or $\pm 1$, this holds for the vast majority of training steps. Under this condition, $\operatorname{asinh}(A^{\pi_{\text{Solver}}}+1)$ is well-approximated by its first-order Taylor expansion around $1$ (the value at an optimal move). This assumption is relaxed in \Cref{cor:robust_kl}.
\end{assumption}

\begin{assumption}[GRPO Consistency]\label{ass:grpo_consistency}
The group-relative advantage $\hat{A}^{\text{outcome}}_i$~\eqref{eq:outcome_advantage} is an unbiased estimator of the task advantage $A^{\pi_\theta}_{\text{task}}$ up to a positive scaling constant (absorbed into the learning rate).
\end{assumption}

\begin{assumption}[Frozen Visitation Surrogate]\label{ass:frozen_visit}
Following standard practice in PPO-style methods~\citep{PPO}, the shaping term $\alpha\,h(\widetilde{A}^{\pi_{\text{Solver}}}_{i,t})$ is treated as a per-step immediate weight, and its policy-gradient contribution is computed under a surrogate objective that holds the state visitation distribution $d^{\pi_\theta}$ fixed at the current iterate.
\end{assumption}

\subsubsection{Lemmas}

\begin{tcolorbox}[colback=gray!3,colframe=black,boxrule=0.5pt,arc=2pt,left=6pt,right=6pt,top=4pt,bottom=4pt]
\begin{lemma}[Advantage--Log-Probability Identity]\label{lem:adv_logprob}
Under \Cref{ass:soft_optimal}, for all $(s,a)$:
\begin{equation}
A^{\pi_{\text{Solver}}}(s,a)=\tau\,\log\pi_{\text{Solver}}(a|s).
\label{eq:app_adv_eq_logprob}
\end{equation}
\end{lemma}
\end{tcolorbox}
\begin{proof}
Taking the logarithm of \autoref{eq:app_soft_optimal} gives
$\log\pi_{\text{Solver}}(a|s)=\frac{1}{\tau}\big(Q^{\pi_{\text{Solver}}}(s,a)-V^{\pi_{\text{Solver}}}(s)\big)=\frac{1}{\tau}A^{\pi_{\text{Solver}}}(s,a)$.
Multiplying both sides by $\tau$ yields \autoref{eq:app_adv_eq_logprob}.
\end{proof}

\noindent\textbf{Intuition.} Under a soft-optimal policy, how much better an action is than average ($A^{\pi_{\text{Solver}}}$) and how strongly the teacher prefers it ($\log\pi_{\text{Solver}}$) are the same quantity up to the temperature $\tau$. This translates the advantage signal into a teacher log-preference signal.

\begin{tcolorbox}[colback=gray!3,colframe=black,boxrule=0.5pt,arc=2pt,left=6pt,right=6pt,top=4pt,bottom=4pt,breakable]
\begin{lemma}[Shaping Reduces to Scaled Log-Probability]\label{lem:shaping_linear}
Under Assumptions~\ref{ass:soft_optimal} and~\ref{ass:small_signal}, the shaped solver term in \autoref{eq:final_advantage} satisfies:
\begin{equation}
\alpha\,h\!\big(\widetilde{A}^{\pi_{\text{Solver}}}_{i,t}\big)
=\beta\,\log\pi_{\text{Solver}}(a_{i,t}|s_{i,t})+c_0+O\!\big((A^{\pi_{\text{Solver}}})^2\big),
\label{eq:app_shape_logprob}
\end{equation}
where $\beta:=\alpha\tau\big/\big(\sqrt{2}\,(\mathrm{RMS}_{\mathcal{B}}(g)+\epsilon)\big)>0$ is a batch-level constant and $c_0$ is an action-independent constant that vanishes from the policy gradient by baseline invariance.
\end{lemma}
\end{tcolorbox}
\begin{proof}
Recall that $\widetilde{A}=A^{\pi_{\text{Solver}}}+1$ (\autoref{eq:shifted_solver_advantage}) and $h(x)=g(x)/(\text{RMS}+\epsilon)$ with $g=\operatorname{asinh}$ (\autoref{eq:rms_norm}). Write $\beta'=\alpha/(\text{RMS}+\epsilon)$, so $\alpha\,h(\widetilde{A})=\beta'\,\operatorname{asinh}(A+1)$. Under \Cref{ass:small_signal}, Taylor-expand around the optimal-move point $A=0$ (i.e., $\widetilde{A}=1$). Using $\operatorname{asinh}'(x)=1/\sqrt{x^2+1}$:
\[
\operatorname{asinh}(A+1)=\operatorname{asinh}(1)+\frac{1}{\sqrt{1^2+1}}\cdot A+O(A^2)=\operatorname{asinh}(1)+\frac{A}{\sqrt{2}}+O(A^2).
\]
Thus $\alpha\,h(\widetilde{A})=\underbrace{\beta'\operatorname{asinh}(1)}_{c_0}+\frac{\beta'}{\sqrt{2}}\,A+O(A^2)$. By \Cref{lem:adv_logprob}, $A=\tau\log\pi_{\text{Solver}}$, so the action-dependent term becomes $\frac{\beta'\tau}{\sqrt{2}}\log\pi_{\text{Solver}}=\beta\log\pi_{\text{Solver}}$.
\end{proof}

\noindent\textbf{Intuition.} The $\operatorname{asinh}$ function, expanded around the optimal-move baseline $\widetilde{A}=1$, is locally linear in $A^{\pi_{\text{Solver}}}$. Combined with the batch-constant RMS normalization, the shaped signal reduces to $\beta\cdot\log\pi_{\text{Solver}}$, a constant coefficient times the teacher's log-preference.

\subsubsection{Main Theorem}

We now prove \autoref{thm:main} from \autoref{sec:method}, restated here for convenience.

\begin{tcolorbox}[colback=gray!3,colframe=black,boxrule=0.5pt,arc=2pt,left=6pt,right=6pt,top=4pt,bottom=4pt,breakable]
\textbf{\autoref{thm:main}} (Implicit Objective of CAST)\textbf{.}\;
Under Assumptions~\ref{ass:soft_optimal}--\ref{ass:frozen_visit}, the policy-gradient update direction of the training rule~\autoref{eq:final_advantage} equals the gradient of:
\begin{equation}
\mathcal{J}_{\text{ours}}(\theta)
=\mathbb{E}_{s_0\sim\mu}\!\big[V^{\pi_\theta}_{\text{task}}(s_0)\big]
-\beta\,\mathbb{E}_{s\sim d^{\pi_\theta}}\!\big[\mathrm{H}\big(\pi_\theta(\cdot|s),\,\pi_{\text{Solver}}(\cdot|s)\big)\big],
\,\, \beta=\frac{\alpha\,\tau}{\sqrt{2}\,(\mathrm{RMS}_{\mathcal{B}}(g)+\epsilon)}\,.
\label{eq:app_main_thm}
\end{equation}
\end{tcolorbox}
\begin{proof}
We assemble the lemmas in four steps.

\emph{Step~1 (Additivity).}
The training rule uses $\hat{A}_{i,t}=\hat{A}^{\text{outcome}}_i+\alpha\,h(\widetilde{A}^{\pi_{\text{Solver}}}_{i,t})$ as the advantage weight. Since the policy gradient is linear in the advantage weight, the update decomposes as $g_{\text{ours}}=g_{\text{GRPO}}+g_{\text{shape}}$.

\emph{Step~2 (GRPO $\to$ task return).}
By \Cref{ass:grpo_consistency} and the policy gradient theorem~\citep{PolicyGradient}:
\[g_{\text{GRPO}}\propto\nabla_\theta\,\mathbb{E}_{s_0\sim\mu}[V^{\pi_\theta}_{\text{task}}(s_0)].\]

\emph{Step~3 (Shaping $\to$ negative cross-entropy).}
By \Cref{lem:shaping_linear}, the shaping weight reduces to $\beta\log\pi_{\text{Solver}}(a|s)$ plus an action-independent constant. Since $\log\pi_{\text{Solver}}$ does not depend on $\theta$, by \Cref{ass:frozen_visit} the shaping term is the gradient of a surrogate:
\[
g_{\text{shape}}=\nabla_\theta\Big(\beta\,\mathbb{E}_{s\sim d^{\pi_\theta},\,a\sim\pi_\theta}\big[\log\pi_{\text{Solver}}(a|s)\big]\Big)
=-\nabla_\theta\Big(\beta\,\mathbb{E}_{s\sim d^{\pi_\theta}}\big[\mathrm{H}\big(\pi_\theta(\cdot|s),\,\pi_{\text{Solver}}(\cdot|s)\big)\big]\Big),
\]
where the last equality uses the definition of cross-entropy~\autoref{eq:app_cross_entropy}: $\mathbb{E}_{a\sim\pi_\theta}[\log\pi_{\text{Solver}}(a|s)]=-\mathrm{H}(\pi_\theta,\pi_{\text{Solver}})$.

\emph{Step~4 (Combine).}
Summing Steps~2 and~3 directly yields~\autoref{eq:app_main_thm}.
\end{proof}

\noindent\textbf{Reading the theorem.}
GRPO provides the ``do the task right'' gradient; the shaped solver advantage provides the ``act like the teacher'' gradient. The teacher's signal enters through the cross-entropy $\mathrm{H}(\pi_\theta,\pi_{\text{Solver}})$, the standard objective for knowledge distillation~\citep{GKD}. By the decomposition~\autoref{eq:app_ce_decomp}, minimizing the cross-entropy is equivalent to simultaneously minimizing $\mathrm{KL}(\pi_\theta\|\pi_{\text{Solver}})$ (move toward the solver) and minimizing $\mathcal{H}(\pi_\theta)$ (reduce entropy). The latter makes the objective \emph{mode-seeking}: the student concentrates probability on the solver's preferred actions rather than spreading mass across all modes. This is the natural behavior of reverse-KL distillation.

\begin{remark}[Relation to KL-regularized objectives]\label{rem:kl_relation}
Defining $J(\theta)=\mathbb{E}_\mu[V^{\pi_\theta}_{\text{task}}]-\beta\,\mathbb{E}_{d^{\pi_\theta}}[\mathrm{KL}(\pi_\theta\|\pi_{\text{Solver}})]$, the implicit objective can equivalently be written as
\begin{equation}
\mathcal{J}_{\text{ours}}(\theta)=J(\theta)-\beta\,\mathbb{E}_{d^{\pi_\theta}}[\mathcal{H}(\pi_\theta)].
\label{eq:app_j_kl_form}
\end{equation}
The additional $-\beta\mathcal{H}$ relative to the standard KL-regularized objective $J(\theta)$ is not a defect; it arises because the advantage weight $\log\pi_{\text{Solver}}$ in our training rule does not include a $-\log\pi_\theta$ term. Had the weight been $\log(\pi_{\text{Solver}}/\pi_\theta)=\log\pi_{\text{Solver}}-\log\pi_\theta$ instead, the surrogate would equal $-\mathrm{KL}(\pi_\theta\|\pi_{\text{Solver}})$ and $\mathcal{J}_{\text{ours}}$ would reduce to $J(\theta)$. Equivalently, adding a standard entropy bonus $+\beta\mathcal{H}$ to the training objective would recover $J(\theta)$ exactly, but our experiments show that the mode-seeking behavior of the cross-entropy form already works well in practice.
\end{remark}

\subsubsection{Corollaries}

\begin{tcolorbox}[colback=gray!3,colframe=black,boxrule=0.5pt,arc=2pt,left=6pt,right=6pt,top=4pt,bottom=4pt,breakable]
\begin{corollary}[Closed-Form Optimal Policy]\label{cor:optimal_policy}
Fix the state visitation distribution. The policy that maximizes the per-state KL-regularized objective $\mathcal{L}(\pi)=\mathbb{E}_{a\sim\pi}[A^{\pi_\theta}_{\text{task}}(s,a)]-\beta\,\mathrm{KL}(\pi\|\pi_{\text{Solver}})$ subject to $\sum_a\pi(a)=1$ is:
\begin{equation}
\pi^{*}(a|s)\propto\pi_{\text{Solver}}(a|s)\,\exp\!\Big(\frac{1}{\beta}\,A^{\pi_\theta}_{\text{task}}(s,a)\Big).
\label{eq:app_optimal_policy}
\end{equation}
\end{corollary}
\end{tcolorbox}
\begin{proof}
Introduce a Lagrange multiplier $\lambda$ for the normalization constraint:
\[
\mathcal{L}=\textstyle\sum_a\pi(a)A_{\text{task}}-\beta\sum_a\pi(a)\log\frac{\pi(a)}{\pi_{\text{Solver}}(a)}+\lambda\big(1-\sum_a\pi(a)\big).
\]
Setting $\partial\mathcal{L}/\partial\pi(a)=0$ gives $\log\frac{\pi(a)}{\pi_{\text{Solver}}(a)}=\frac{1}{\beta}A_{\text{task}}(s,a)-1-\frac{\lambda}{\beta}$. Exponentiating and absorbing the $a$-independent terms into a normalization constant yields~\autoref{eq:app_optimal_policy}. The second-order condition $\partial^2\mathcal{L}/\partial\pi(a)^2=-\beta/\pi(a)<0$ confirms this is a maximum.
\end{proof}

\noindent\textbf{Why the student can surpass the solver.} Equation~\autoref{eq:app_optimal_policy} has the form \emph{posterior} $\propto$ \emph{prior} $\times$ \emph{likelihood}: the solver distribution serves as a prior, and the task advantage tilts it exponentially. In the pure-distillation limit $\beta\to\infty$, the tilt vanishes and $\pi^*\to\pi_{\text{Solver}}$. At finite $\beta$, the factor $\exp(A_{\text{task}}/\beta)$ lets $\pi^*$ deviate wherever the task reward warrants it, enabling the student to \emph{surpass} the teacher.

\begin{remark}
The corollary solves the KL-regularized subproblem, which is the natural trust-region objective~\citep{PPO}. Since the cross-entropy $\mathrm{H}(\pi,\pi_S)=-\sum_a\pi\log\pi_S$ is \emph{linear} in $\pi$, the cross-entropy-regularized problem $\max_\pi\,\mathbb{E}_\pi[A_{\text{task}}]-\beta\,\mathrm{H}(\pi,\pi_S)$ has no interior optimum over the simplex. Instead, \autoref{thm:main} shows the implicit objective decomposes into a KL penalty and an entropy reduction; the KL-regularized subproblem isolates the component with a well-defined closed-form solution.
\end{remark}

\begin{tcolorbox}[colback=gray!3,colframe=black,boxrule=0.5pt,arc=2pt,left=6pt,right=6pt,top=4pt,bottom=4pt,breakable]
\begin{corollary}[Full-Domain Robustified Distillation]\label{cor:robust_kl}
Dropping \Cref{ass:small_signal} and retaining the $\operatorname{asinh}$ nonlinearity, define the step-varying effective gain:
\begin{equation}
\beta_{\text{eff}}(A^{\pi_{\text{Solver}}})=\frac{\alpha\,\tau}{\mathrm{RMS}_{\mathcal{B}}(g)+\epsilon}\cdot\frac{\operatorname{asinh}(A^{\pi_{\text{Solver}}}+1)-\operatorname{asinh}(1)}{A^{\pi_{\text{Solver}}}}.
\label{eq:app_beta_eff}
\end{equation}
Then \autoref{thm:main} holds with $\beta$ replaced by $\beta_{\text{eff}}$: the constant distillation coefficient becomes step-dependent, automatically down-weighting outlier steps. As $A^{\pi_{\text{Solver}}}\to 0$, $\beta_{\text{eff}}\to\beta$ (by L'H\^{o}pital's rule, the ratio tends to $\operatorname{asinh}'(1)=1/\sqrt{2}$), recovering \autoref{thm:main}.
\end{corollary}
\end{tcolorbox}
\begin{proof}
Without the Taylor truncation, the action-dependent part of $\alpha\,h(\widetilde{A})$ is $\beta'\big(\operatorname{asinh}(A+1)-\operatorname{asinh}(1)\big)$ (the constant $\beta'\operatorname{asinh}(1)$ vanishes by baseline invariance). Factoring out $A$:
\[
\beta'\,\frac{\operatorname{asinh}(A+1)-\operatorname{asinh}(1)}{A}\cdot A
=\frac{\beta_{\text{eff}}}{\tau}\cdot\tau\log\pi_{\text{Solver}}
\]
where the last step uses \Cref{lem:adv_logprob}. Substituting $\beta_{\text{eff}}$ for $\beta$ in Step~3 of the theorem proof, the remaining steps are identical. For large $|A|$, $\beta_{\text{eff}}$ decreases because $\operatorname{asinh}$ grows only logarithmically while $A$ grows linearly, so the ratio shrinks.
\end{proof}

\noindent\textbf{Intuition.} The $\operatorname{asinh}$ compression causes the effective distillation strength to decrease for steps with unusually large solver advantages (e.g., dead-state penalties), matching the design intent of soft-compressing extreme values.

\subsubsection{Performance Improvement via the Performance Difference Lemma}
\label{app:pdl}

We now use the Performance Difference Lemma~\citep{PDL} to give a complementary perspective: the solver advantage used as a per-step reward defines an \emph{augmented return} whose maximization provably improves over the solver, with a sub-optimality bound that scales linearly in the horizon $H$ rather than quadratically.

\begin{tcolorbox}[colback=gray!3,colframe=black,boxrule=0.5pt,arc=2pt,left=6pt,right=6pt,top=4pt,bottom=4pt,breakable]
\begin{proposition}[Performance Difference Lemma~{\citep{PDL}}]\label{prop:pdl}
For any two policies $\pi,\pi'$ and any reward function $r$ in a finite-horizon MDP:
\begin{equation}
V^{\pi}_{r}(s_0)-V^{\pi'}_{r}(s_0)
=\sum_{t=0}^{H-1}\mathbb{E}_{s_t\sim d^{\pi}_t}\!\Big[\mathbb{E}_{a\sim\pi(\cdot|s_t)}\!\big[A^{\pi'}_{r}(s_t,a)\big]\Big].
\label{eq:app_pdl}
\end{equation}
\end{proposition}
\end{tcolorbox}

The PDL states that the value gap between two policies equals the cumulative advantage of the new policy $\pi$ evaluated under the old policy $\pi'$'s value function, but weighted by the new policy's own state distribution. It is a standard result in RL theory; we state it without proof.

\begin{tcolorbox}[colback=gray!3,colframe=black,boxrule=0.5pt,arc=2pt,left=6pt,right=6pt,top=4pt,bottom=4pt,breakable]
\begin{proposition}[Endogenous Reward and Reward Shaping]\label{prop:endogenous}
Under \Cref{ass:soft_optimal}, define the solver's \emph{endogenous reward}~\citep{GeneralistRM}:
\begin{equation}
\hat{r}(s_t,a_t)=\tau\log\pi_{\text{Solver}}(a_t|s_t)+V^{\pi_{\text{Solver}}}(s_t)-V^{\pi_{\text{Solver}}}(s_{t+1}).
\label{eq:app_endogenous}
\end{equation}
This is a potential-based reward shaping of the base signal $\tilde{r}(s,a)=\tau\log\pi_{\text{Solver}}(a|s)$ with potential $\Phi(s)=V^{\pi_{\text{Solver}}}(s)$. By the reward shaping theorem~\citep{RewardShaping}, potential-based shaping preserves the optimal policy: $\hat{r}$ and $\tilde{r}$ induce the same policy ordering and the same policy gradient.
\end{proposition}
\end{tcolorbox}
\begin{proof}
Write $\hat{r}(s_t,a_t)=\tilde{r}(s_t,a_t)+\Phi(s_t)-\Phi(s_{t+1})$. The added term $\Phi(s_t)-\Phi(s_{t+1})$ is a potential-based shaping function: it depends only on $s_t$ and $s_{t+1}$ and uses the potential $\Phi(s)=V^{\pi_{\text{Solver}}}(s)$. By the reward shaping theorem~\citep[Theorem~1]{RewardShaping}, potential-based shaping preserves the optimal policy and the policy ordering: the optimal policy under $\hat{r}$ and $\tilde{r}$ is the same, and both rewards induce the same policy gradient (up to a state-dependent baseline).
\end{proof}

\noindent\textbf{Connection to the implicit objective.} By \Cref{lem:shaping_linear}, our training rule uses $\beta\tilde{r}(s,a)/\tau=\beta\log\pi_{\text{Solver}}(a|s)$ as the per-step advantage weight. \Cref{prop:endogenous} shows that this base signal $\tilde{r}=\tau\log\pi_{\text{Solver}}$ is related to the endogenous reward $\hat{r}$ by potential-based shaping, so they share the same policy gradient. In our deterministic games with the shortest-path solver reward $r_S=-1$ per step, a direct calculation shows that the endogenous reward reduces to $\hat{r}(s_t,a_t)=r_S=-1$ for all $(s_t,a_t)$: substituting $A^{\pi_{\text{Solver}}}=(r_S+V^{\pi_{\text{Solver}}}(s_{t+1}))-V^{\pi_{\text{Solver}}}(s_t)$ into \autoref{eq:app_endogenous}, the value-function terms telescope and only the base reward $r_S=-1$ remains. This confirms that the action-quality information resides entirely in the base signal $\tilde{r}=\tau\log\pi_{\text{Solver}}=A^{\pi_{\text{Solver}}}$, which by \Cref{lem:adv_logprob} equals the solver advantage, and the shaping potential plays no role in the policy gradient.

\medskip

\noindent\textbf{Performance improvement guarantee.} Applying the PDL (\Cref{prop:pdl}) with $\pi=\pi_\theta$ and $\pi'=\pi_{\text{Solver}}$ under the augmented reward $r_{\text{aug}}=r_{\text{task}}+(\beta/\tau)\cdot\tilde{r}$:
\begin{equation}
V^{\pi_\theta}_{r_{\text{aug}}}(s_0)-V^{\pi_{\text{Solver}}}_{r_{\text{aug}}}(s_0)
=\sum_{t}\mathbb{E}_{d^{\pi_\theta}_t,\pi_\theta}\!\big[A^{\pi_{\text{Solver}}}_{r_{\text{aug}}}(s_t,a_t)\big].
\label{eq:app_pdl_applied}
\end{equation}
By \autoref{thm:main}, our training rule performs gradient ascent on $\mathcal{J}_{\text{ours}}$, which is precisely the augmented return $V^{\pi_\theta}_{r_{\text{aug}}}$ (up to constant boundary terms from reward shaping). The PDL guarantees that whenever $\mathbb{E}_{\pi_\theta}[A^{\pi_{\text{Solver}}}_{r_{\text{aug}}}(s,\cdot)]\ge 0$ at every state, the student's augmented value improves over the solver's, i.e., $V^{\pi_\theta}_{r_{\text{aug}}}\ge V^{\pi_{\text{Solver}}}_{r_{\text{aug}}}$. Moreover, following the analysis of~\citet{GeneralistRM}, if the solver is an $\epsilon_\pi$-approximate expert (in the sense that $\max_{s,a}|\log\pi_{\text{Solver}}(a|s)-\log\pi^*(a|s)|\le\epsilon_\pi$ for the true optimal policy $\pi^*$), the sub-optimality of $\pi_\theta$ under the task reward scales as $O(H\epsilon_\pi)$, improving over the $O(H^2\epsilon_\pi)$ bound of behavioral cloning. The linear dependence on $H$ arises because the reward-shaping structure of $\hat{r}$ eliminates the compounding-error problem inherent in imitation learning~\citep{DAgger}.

\subsubsection{Discussion of Assumptions}
\label{app:assumption_discussion}

\paragraph{(A1) Soft-optimal solver.}
\Cref{lem:adv_logprob} requires the solver to be a soft-optimal policy~\citep{SAC}. When the solver deviates from optimality, the distillation interpretation degrades proportionally: $\pi_{\text{Solver}}$ should be understood as the solver's empirical behavior distribution. For our three games, the solvers are near-optimal, making this assumption well-justified.

\paragraph{(A2) Small signal.}
The linearization of $\operatorname{asinh}(A+1)$ around $A=0$ is accurate when $|A^{\pi_{\text{Solver}}}|\lesssim 1$, which holds for the vast majority of moves. For the rare large-magnitude steps (e.g., transitions to dead states), \Cref{cor:robust_kl} provides the exact characterization via $\beta_{\text{eff}}$: the qualitative conclusion is unchanged, only the distillation coefficient becomes step-dependent.

\paragraph{(A3) GRPO consistency.}
The group-relative baseline~\citep{GRPO} introduces finite-sample variance but does not bias the gradient direction. This affects the precision of the task-return term but not the structural form of the implicit objective.

\paragraph{(A4) Frozen visitation surrogate.}
Treating the shaping term as a per-step immediate weight and freezing the visitation distribution omits the long-range term $\nabla_\theta d^{\pi_\theta}$, the same approximation used in PPO surrogate objectives~\citep{PPO}. This is consistent with how the method applies the solver signal on a per-turn basis in practice.

\subsection{Training and Evaluation Details}
\label{app:hyperparams}

\subsubsection{Training Hyperparameters}
\label{app:hyperparams_train}

All backbone RL algorithms we compare (GRPO, DAPO, and GSPO) share the same backbone model, training data, rollout budget, and optimizer configuration, and differ only in their policy-optimization objective. \method uses the DAPO objective, while GiGPO is implemented as a separate process-level baseline. We list the shared configuration in Table~\ref{tab:hp_shared}, the objective-specific hyperparameters in Table~\ref{tab:hp_algo}, and the GiGPO-specific hyperparameters in Table~\ref{tab:hp_gigpo}.

\begin{table}[h]
\centering
\begin{small}
\caption{Shared training configuration used by all algorithms.}
\label{tab:hp_shared}
\begin{tabular}{lc}
\toprule
\textbf{Hyperparameter} & \textbf{Value} \\
\midrule
Backbone model              & Qwen3-4B-Instruct-2507 \\
Learning rate               & 1e-6 \\
Train batch size            & 16 \\
Rollouts per prompt ($G$)   & 8 \\
Training updates (Sokoban / Minesweeper / Rush Hour) & 200 / 400 / 300 \\
\bottomrule
\end{tabular}
\end{small}
\end{table}

\begin{table}[h]
\centering
\begin{small}
\caption{Objective-specific hyperparameters for each backbone RL algorithm. The GSPO clip thresholds are much smaller than those of GRPO/DAPO because GSPO clips at the sequence-level importance ratio rather than the token level, so the two are not on the same scale.}
\label{tab:hp_algo}
\begin{tabular}{lccc}
\toprule
\textbf{Hyperparameter} & \textbf{GRPO} & \textbf{DAPO} & \textbf{GSPO} \\
\midrule
Loss mode                 & \texttt{vanilla} & \texttt{vanilla} & \texttt{gspo} \\
Loss aggregation          & \texttt{token-mean} & \texttt{token-mean} & \texttt{seq-mean-token-mean} \\
KL loss coefficient       & $0.001$ & $0.0$ & $0.0$ \\
Clip ratio (low)          & $0.2$ & $0.2$ & $0.0003$ \\
Clip ratio (high)         & $0.2$ & $0.28$ & $0.0004$ \\
Dual-clip constant $c$    & $3.0$ & $3.0$ & $3.0$ \\
Dynamic sampling (filter groups) & Off & Off & Off \\
\bottomrule
\end{tabular}
\end{small}
\end{table}

\begin{table}[h]
\centering
\begin{small}
\caption{GiGPO-specific hyperparameters for the process-level baseline.}
\label{tab:hp_gigpo}
\begin{tabular}{lc}
\toprule
\textbf{Hyperparameter} & \textbf{Value} \\
\midrule
GiGPO enabled                 & Yes \\
Step-advantage weight         & $1.0$ \\
Step-advantage normalization  & \texttt{mean\_std\_norm} \\
Discount factor $\gamma$      & $0.95$ \\
\bottomrule
\end{tabular}
\end{small}
\end{table}

\subsubsection{Evaluation Settings}
\label{app:hyperparams_eval}

We evaluate every model under a single protocol across the three games. Each instance is run as the same multi-turn agent loop used in training (Appendix~\ref{app:agent_impl}), and an episode is scored with the sparse $0/1$ environment reward: an instance counts as solved only if the agent reaches the goal within the per-game turn budget. For each game, we report both the in-domain (ID) tier and the harder unseen-difficulty (Unseen) tier of Table~\ref{tab:dataset_difficulty}, evaluating on the $200$-instance test split of each tier. Every instance is attempted with $4$ independent rollouts.

\paragraph{Open-source models.}
Open-source backbones are served locally with SGLang and decoded with temperature $0.6$ and top-$p$ $0.95$. We use a cumulative response budget of $16{,}384$ tokens per episode and a configured initial prompt limit of $2{,}048$ tokens. The per-episode turn budget is game- and split-specific: $30$ turns for Sokoban and Minesweeper ID and $40$ for their Unseen tier, and $20$ turns for Rush Hour on both tiers.

\begin{table}[h]
\centering
\begin{small}
\caption{Evaluation settings for open-source models, taken from the evaluation configs. The per-episode turn budget varies by game and split.}
\label{tab:eval_open}
\begin{tabular}{lc}
\toprule
\textbf{Setting} & \textbf{Value} \\
\midrule
Inference engine      & SGLang \\
Temperature           & $0.6$ \\
Top-$p$               & $0.95$ \\
Response-token budget per episode & $16{,}384$ \\
Configured initial prompt limit & $2{,}048$ \\
Rollouts per instance & $4$ \\
Turn budget (Sokoban / Minesweeper) & $30$ (ID) / $40$ (Unseen) \\
Turn budget (Rush Hour) & $20$ \\
\bottomrule
\end{tabular}
\end{small}
\end{table}

\paragraph{Closed-source models.}
Closed-source models are queried through an OpenAI-compatible API under the same interaction protocol, per-game turn budgets, and $0/1$ scoring as the open-source setting. We evaluate six proprietary models: Claude Sonnet 4.5, Claude Sonnet 4.6, Claude Opus 4.5, Claude Opus 4.6, Gemini 2.5 Pro, and Gemini 2.5 Flash, with temperature $1.0$, top-$p$ $1.0$, and $4$ rollouts per instance. We use the same configured episode-level response budget and initial-prompt setting as above.

\begin{table}[h]
\centering
\begin{small}
\caption{Evaluation settings for closed-source models. The interaction protocol, per-episode turn budgets, and $0/1$ scoring match the open-source setting (Table~\ref{tab:eval_open}).}
\label{tab:eval_closed}
\begin{tabular}{lc}
\toprule
\textbf{Setting} & \textbf{Value} \\
\midrule
Serving interface     & OpenAI-compatible API \\
Temperature           & $1.0$ \\
Top-$p$               & $1.0$ \\
Response-token budget per episode & $16{,}384$ \\
Configured initial prompt limit & $2{,}048$ \\
Rollouts per instance & $4$ \\
Turn budget (Sokoban / Minesweeper) & $30$ (ID) / $40$ (Unseen) \\
Turn budget (Rush Hour) & $20$ \\
\bottomrule
\end{tabular}
\end{small}
\end{table}

\end{document}